\documentclass[twocolumn, switch]{article} 

\usepackage{preprint}

\usepackage{amsmath, amsthm, amssymb, amsfonts}

\usepackage[numbers,square]{natbib}
\usepackage[utf8]{inputenc}	
\usepackage[T1]{fontenc}	
\usepackage{xcolor}		
\usepackage[colorlinks = true,
            linkcolor = purple,
            urlcolor  = blue,
            citecolor = cyan,
            anchorcolor = black]{hyperref}	
\usepackage{booktabs} 		
\usepackage{nicefrac}		
\usepackage{microtype}		
\usepackage{lineno}		
\usepackage{float}			

\usepackage{lipsum}		

\usepackage{amsmath,amsfonts}
\usepackage{algorithmic}
\usepackage{algorithm}
\usepackage{graphicx}

\usepackage{multirow}
\usepackage{booktabs}
\usepackage{graphicx}
\usepackage{xcolor}
\usepackage{amsmath,amsfonts,amsthm}
\usepackage{mathtools}
\usepackage{bibentry}

\newtheorem{theorem}{Theorem}
\newtheorem{lemma}{Lemma}

\usepackage{newfloat}
\DeclareFloatingEnvironment[name={Supplementary Figure}]{suppfigure}
\usepackage{sidecap}
\sidecaptionvpos{figure}{c}

\usepackage{titlesec}
\titlespacing\section{0pt}{12pt plus 3pt minus 3pt}{1pt plus 1pt minus 1pt}
\titlespacing\subsection{0pt}{10pt plus 3pt minus 3pt}{1pt plus 1pt minus 1pt}
\titlespacing\subsubsection{0pt}{8pt plus 3pt minus 3pt}{1pt plus 1pt minus 1pt}

\usepackage{tikz,xcolor,hyperref}

\definecolor{lime}{HTML}{A6CE39}
\DeclareRobustCommand{\orcidicon}{
	\begin{tikzpicture}
	\draw[lime, fill=lime] (0,0)
	circle [radius=0.16]
	node[white] {{\fontfamily{qag}\selectfont \tiny ID}};
	\draw[white, fill=white] (-0.0625,0.095)
	circle [radius=0.007];
	\end{tikzpicture}
	\hspace{-2mm}
}
\foreach \x in {A, ..., Z}{\expandafter\xdef\csname orcid\x\endcsname{\noexpand\href{https://orcid.org/\csname orcidauthor\x\endcsname}
			{\noexpand\orcidicon}}
}

\title{Embracing Federated Learning: Enabling Weak Client Participation via Partial Model Training}

\author{%
  Sunwoo Lee\textsuperscript{\rm 1}, Tuo Zhang\textsuperscript{\rm 2}, Saurav Prakash\textsuperscript{\rm 3}, Yue Niu\textsuperscript{\rm 2}, Salman Avestimehr\textsuperscript{\rm 2} \\
  \textsuperscript{\rm 1}Inha University, Republic of Korea\\
  \textsuperscript{\rm 2}University of Southern California, USA\\
  \textsuperscript{\rm 3}University of Illinois Urbana-Champaign, USA\\
  \textsuperscript{\rm 1}\texttt{sunwool@inha.ac.kr} \\
  \textsuperscript{\rm 2}\texttt{\{tuozhang, yueniu, avestime\}@usc.edu} \\
  \textsuperscript{\rm 3}\texttt{sauravp2@illinois.edu}
}

\begin{document}

\twocolumn[ 
  \begin{@twocolumnfalse} 

\maketitle

\begin{abstract}
In Federated Learning (FL), clients may have weak devices that cannot train the full model or even hold it in their memory space. To implement large-scale FL applications, thus, it is crucial to develop a distributed learning method that enables the participation of such weak clients. We propose \texttt{EmbracingFL}, a general FL framework that allows all available clients to join the distributed training regardless of their system resource capacity. The framework is built upon a novel form of partial model training method in which each client trains as many consecutive output-side layers as its system resources allow. Our study demonstrates that \texttt{EmbracingFL} encourages each layer to have similar data representations across clients, improving FL efficiency. The proposed partial model training method guarantees convergence to a neighbor of stationary points for non-convex and smooth problems. We evaluate the efficacy of \texttt{EmbracingFL} under a variety of settings with a mixed number of strong, moderate ($\sim 40\%$ memory), and weak ($\sim 15\%$ memory) clients, datasets (CIFAR-10, FEMNIST, and IMDB), and models (ResNet20, CNN, and LSTM). Our empirical study shows that \texttt{EmbracingFL} consistently achieves high accuracy as like all clients are strong, outperforming the state-of-the-art width reduction methods (i.e. HeteroFL and FjORD).
\end{abstract}
\vspace{0.35cm}

  \end{@twocolumnfalse} 
] 


\section{Introduction}

To adopt Federated Learning (FL) \cite{mcmahan2017communication} in real-world applications, enabling weak devices to participate in distributed training is crucial.
One common assumption in FL is that all individual clients train a local model that is the same as the global model \cite{li2020federated,diao2020heterofl,li2021survey,liu2022no,zhang2022federated}.
However, many edge devices such as mobile phones and IoT devices likely have heterogeneous system resources, and such an assumption is not practical in FL environments.
If the model is large, for example, some weak devices may not even be able to hold the full model in the memory space and cannot join the training.
Therefore, to exploit all the available client-side data, enabling weak client participation is critical.

Several recent works proposed FL strategies that enable clients to train different local models.
Knowledge distillation techniques \cite{sodhani2020closer,bistritz2020distributed,lin2020ensemble,cho2023communication} and FedHe \cite{chan2021fedhe} allow the clients to have different models and exchange only their output logits.
However, they work only when the number of clients is small or the model is personalized.
FedHM \cite{yao2021fedhm} employs a low-rank approximation technique to support various local model sizes.
While it shows a promising heterogeneous FL performance, it suffers from the expensive and frequent model factorizations on the server side.
FedPT \cite{sidahmed2021efficient} shows that freezing a large portion of the model and training only the rest of it still achieves good accuracy.
This approach requires all clients to reconstruct the same full model from a pre-defined random seed.
HeteroFL \cite{diao2020heterofl} and FjORD \cite{horvath2021fjord} commonly reduce the width of each network layer and assign such small models to weak clients.
While this \textit{width reduction} method alleviates the weak clients' workload, they suffer from a substantial accuracy drop when there are many weak clients.

\begin{figure*}[t]
\centering
\includegraphics[width=1.8\columnwidth]{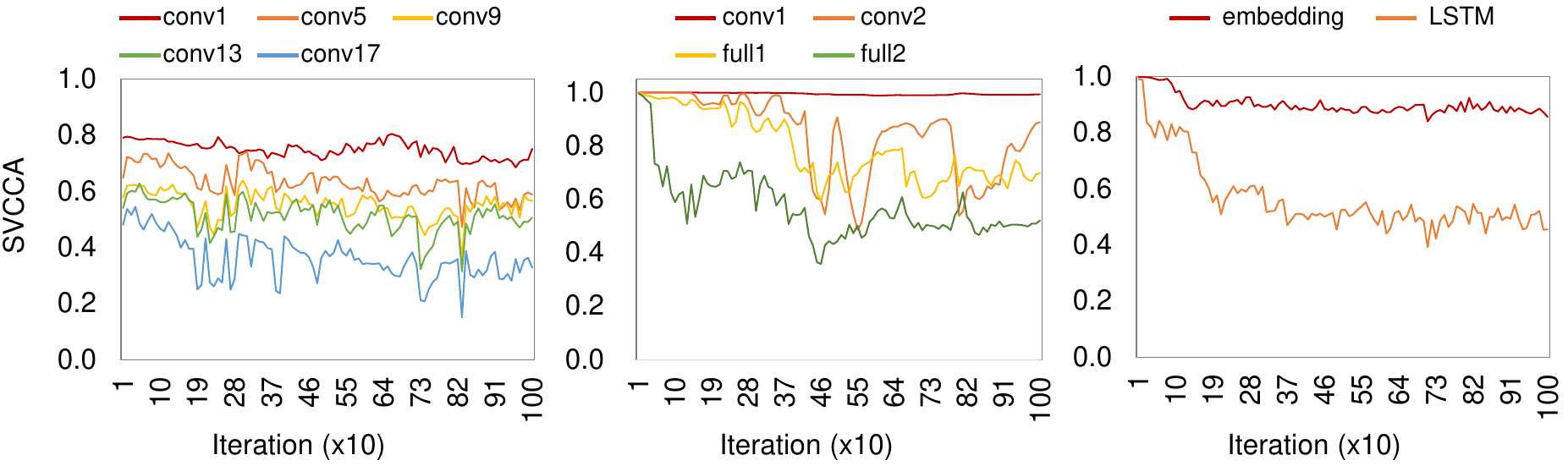}
\caption{
    The layer-wise maximum SVCCA (data representation similarity) among 128 clients.
    The SVCCA is measured from CIFAR-10 training of ResNet20 (left), FEMNIST training of CNN (middle), and IMDB training of LSTM (right).
    All 128 local models are independently trained for 1,000 iterations without synchronizations.
}
\label{fig:svcca}
\end{figure*}

We find that the core problem these previous works try to address can be generalized as follows:
\textit{``If a client is too weak to train the full model, how could it contribute to the training?''}.
The answer to this fundamental question will allow us to design an effective and practical FL solution that exploits heterogeneous systems.

One intriguing finding of our empirical study is that neural network layers have distinguishable patterns in their output data.
Figure \ref{fig:svcca} shows layer-wise data representation similarity across clients measured from several benchmarks\footnote{The CIFAR-10 \cite{krizhevsky2009learning} and IMDB \cite{maas-EtAl:2011:ACL-HLT2011} datasets are non-IID based on the label values (Dirichlet distribution with $\alpha=0.1$). We also use LEAF version of FEMNIST \cite{caldas2018leaf} which is already provided as non-IID.}.
We quantify the similarity using Singular Vector Canonical Correlation Analysis (SVCCA) \cite{raghu2017svcca} that is known to effectively reveal the similarity of two given matrices (See Appendix).
First, we let 128 clients independently train their models without any model aggregations.
We periodically calculate SVCCA among all possible pairs of clients and find the maximum value.
The SVCCA is calculated using the same data samples not included in any local training datasets.
We see that the input-side layers show higher SVCCA values than the output-side layers in CIFAR-10, FEMNIST, and IMDB review classification experiments.
Furthermore, the first few layers consistently show a high degree of similarity during the training even without any synchronizations.
These observations provide a critical insight into which layers learn more `similar' knowledge across the clients than the other layers.

Motivated by the above observations, we propose \texttt{EmbracingFL}, a general FL framework that leverages a layer-wise partial training method to enable weak client participation.
The core idea is to assign an output-side consecutive subset of network layers to the weak clients and let them perform the backward pass within the assigned layers only.
Since the local models mostly learn similar knowledge across the clients at the input-side layers, as shown in Figure \ref{fig:svcca}, the weak clients can most effectively contribute to the global model training when they focus on the output-side layers.
To the best of our knowledge, our study explains why the output-side layers should be assigned to weak clients and analyzes the benefits of such an approach for the first time.
We support our arguments by analyzing the impact of various partial model synchronization schemes on the data representation similarity.
\texttt{EmbracingFL} also enables us to directly average the local models across the clients because all the clients always view the same model architecture.

Our theoretical analysis shows that \texttt{EmbracingFL} guarantees convergence to a neighbor region of a stationary point for non-convex smooth problems under a piece-wise Lipschitz continuity assumption.
The analysis examines the trade-off between the reduced workload at the weak client and the impact on the convergence rate of the global loss.
The key result is that \texttt{EmbracingFL} converges to a close neighborhood of a minimum regardless of the number of weak clients and how many layers are assigned to them.
This solution bias is caused by the input-side layers that are trained by the strong clients only.
A similar analysis for strongly convex problems is shown in \cite{cho2020client}.
This performance guarantee allows weak clients to flexibly choose how many layers to train based on their available system resources.

\texttt{EmbracingFL} shows remarkably improved FL performance on heterogeneous clients as compared to the state-of-the-art width reduction methods (HeteroFL and FjORD).
Under a variety of settings with a mixed number of strong, moderate ($\sim 40\%$ memory footprint), and weak ($\sim 15\%$ memory footprint) clients, datasets (CIFAR-10, FEMNIST \cite{caldas2018leaf}, and IMDB review), and models (ResNet20, CNN, and LSTM \cite{hochreiter1997long}), \texttt{EmbracingFL} consistently achieves high accuracy as like all the clients are strong.
For example, when $50\%$ of clients are weak, \texttt{EmbracingFL} shows CIFAR-10 validation accuracy close to the strong client-only accuracy ($< 0.1\%$ difference) while the state-of-the-art width reduction methods achieve $7.1\%$ lower accuracy.
The accuracy improvement becomes even more substantial as the fraction of weak clients increases ($22.6\%$ higher accuracy when $87.5\%$ of clients are weak, See Appendix).
We also discuss other benefits of the proposed method, such as the reduced backward pass time on real edge devices and the resilience to inaccurate batch normalization statistics.
\section {Related Work}
\textbf{Federated Learning on Heterogeneous Clients} --
A few recent research works discussed how to utilize heterogeneous clients in FL.
Knowledge distillation techniques \cite{sodhani2020closer,bistritz2020distributed,lin2020ensemble,cho2023communication} and FedHe \cite{chan2021fedhe} enable the clients to train different local models and exchange their knowledge (output logits).
However, this approach effectively works only when the number of clients is small or the model is personalized to each client.
HeteroFL \cite{diao2020heterofl} and FjORD \cite{horvath2021fjord} directly reduce the width of every layer for weak clients.
While this static parameter dropping improves the accuracy compared to the random parameter dropping \cite{wen2022federated}, the accuracy loss is still not negligible.

\begin{figure*}[t]
\centering
\includegraphics[width=1.7\columnwidth]{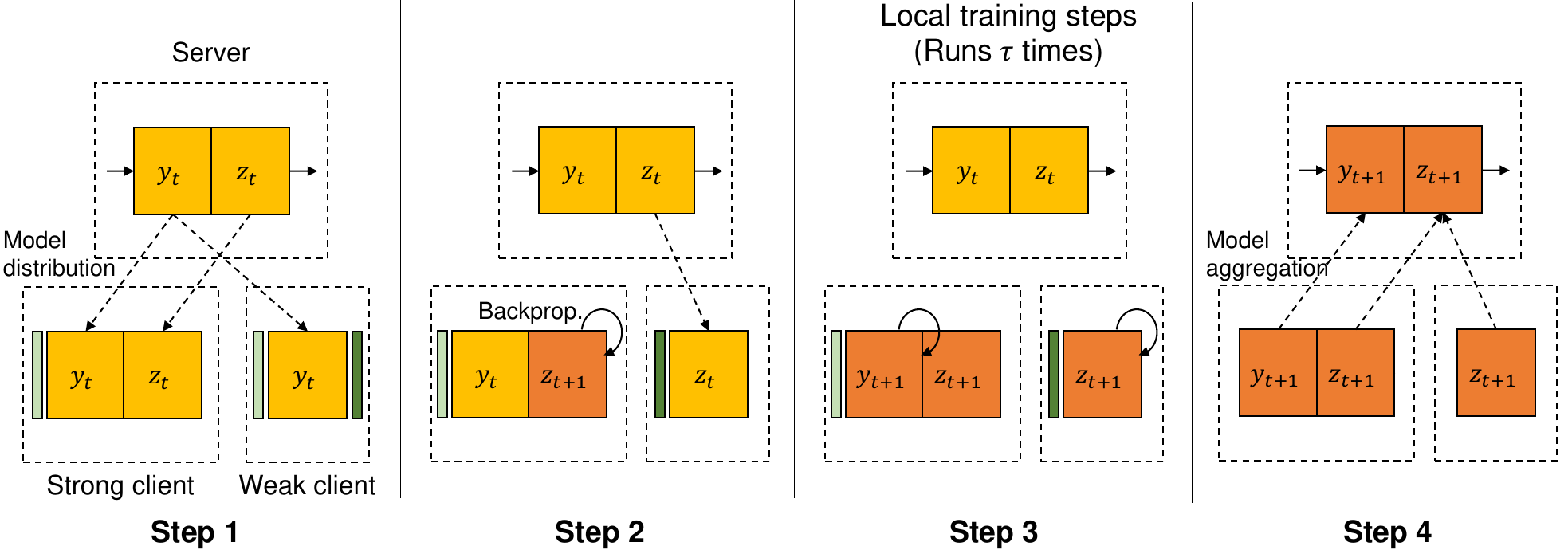}
\caption{
    The schematic illustration of \texttt{EmbracingFL}. While the \textit{strong} client trains the whole model parameters, the \textit{weak} client trains an output-side subset of layers only. The \textit{weak} client can determine how many layers to train based on its available resource capacity.
}
\label{fig:schematic}
\end{figure*}

Several model decomposition methods have also been studied to tackle the heterogeneous system issue in FL.
FedHM \cite{yao2021fedhm} reduces the communication and computational cost on the client side by assigning a model with a reduced rank.
However, FedHM was evaluated using a small number of clients (20), and the accuracy improvement over HeteroFL is smaller than $1.5\%$.
More recently, Mei et al. proposed FLANC, a general model decomposition framework for FL \cite{mei2022resource}.
While this previous work shows promising results, the performance of FLANC under extremely heterogeneous FL environments is still not guaranteed.
For example, the authors evaluate the performance of FLANC when only up to $25\%$ of clients are weak (their resource capacity is $25\%$ of the strong devices).

Some other FL methods also consider heterogeneous clients.
Chen et al. consider heterogeneous system resources across devices and propose an asynchronous model aggregation scheme to utilize those different IoT devices \cite{chen2021towards}.
However, they still assume that every device is strong enough to individually train the target model.
InclusiveFL \cite{liu2022no} is based on a similar principle as our proposed method, which reduces the depth of networks for weak clients.
The key difference between InclusiveFL and our method is the layer selection criteria.
Specifically, they remove the output-side layers for the weak clients.
Interestingly, our study finds that the output-side layers learn more critical information while this previous work argues the opposite.
We will discuss this in Section \ref{sec:impact}.

\textbf{Reducing Client-Side Workload in Federated Learning} --
ResIST \cite{dun2021resist} proposes a local SGD-based distributed learning method that reduces the client-side computational workload.
The model is split into several subsets of consecutive layers, and every client randomly samples a subset and trains it locally.
While this sub-model shuffling approach reduces the client-side computational cost, its effectiveness was not validated in a realistic FL scale.
The performance was evaluated using up to 8 clients only.
FedPara \cite{hyeon2021fedpara} employs low-rank factorization to reduce the model size.
This SOTA factorization method reduces the communication cost, however, the computational cost and the memory footprint are not reduced.
The frequent model factorizations can also cause an extra computational cost on the server side.
SplitFed \cite{thapa2020splitfed} combines FL and Split Learning \cite{gupta2018distributed} to reduce the client-side workload.
It assumes there exists an extra server that can calculate the partial model gradient.
However, the distributed backprop. potentially has an expensive latency cost caused by frequent communications in edge device environments.
In addition, direct activation exchanges may have privacy issues under typical FL environments.
FedPrune \cite{munir2021fedprune} proposes a random pruning of the network to facilitate weak clients.
While this approach reduces the workload for weak clients, it is rather data-driven and the resource capacity is not considered when pruning the network.
To exploit many different edge devices in real-world FL environments, an FL strategy must be aware of the available system resources and determine the sub-model size.

\textbf{Layer-Wise Model Training} --
A few studies have demonstrated the benefits of a layer-wise partial training approach.
Belilovsky et al. adopt the greedy layer-wise training method \cite{bengio2006greedy} to modern image classification tasks \cite{belilovsky2019greedy}.
This work introduces a promising scaling method that builds up a large model based on several shallow models.
Model freezing methods \cite{he2021pipetransformer,yang2023efficient} and progressive learning methods \cite{li2021progressive,li2022automated} are well aligned with this approach.
However, the efficacy of these existing methods has not been validated in the FL context.
FedMA employs a matched averaging method to build up a large model by averaging layers with common features \cite{wang2020federated}.
Ma et al. propose a layer-wise model aggregation method for personalized Federated Learning \cite{ma2022layer}.
FedLAMA is a communication-efficient layer-wise distributed learning method \cite{lee2023layer}.
Lee et al. discussed the impact of the partial model synchronization on the convergence rate in the FL context \cite{lee2023partial}.
All these existing works commonly show how to effectively leverage a partial model training method in Federated Learning.
However, they do not consider the system heterogeneity issue that is common in realistic FL environments.
The aforementioned FLANC enables clients to train a local model with a different width \cite{mei2022resource}.
FLANC decomposes each layer into two subsequent layers, and the first layer has the same dimensions across all the clients while the second layer has client-specific dimensions.
The second layers are synchronized only among the clients with the same capacity.
We will discuss the impact of such partial model synchronizations in Section \ref{sec:impact}.

\section {Embracing Federated Learning Framework}
In this section, we describe \texttt{EmbracingFL} framework, the layer-wise partial model training strategy for large-scale FL on heterogeneous systems.
We define \lq{}\textit{strong}\rq{} client as an edge device in which the full model can be locally trained independently of other clients.
We also define \lq{}\textit{weak}\rq{} client as an edge device in which the full model cannot be effectively trained due to several possible reasons such as limited memory space or weak computing power.
For simplicity, we define the layers assigned to the weak clients as $\mathbf{z}$ and all the other layers as $\mathbf{y}$.
The full model can be obtained by concatenating these two sub-models: $\mathbf{x} = (\mathbf{y}, \mathbf{z})$.

\subsection {Layer-wise Partial Training Strategy}

We describe \texttt{EmbracingFL} focusing on how to enable \textit{weak} clients to join the global model training.
The goal of \texttt{EmbracingFL} is to enable \textit{weak} clients not to hold the full model in its memory space at any moment.
Figure \ref{fig:schematic} depicts the schematic illustration of \texttt{EmbracingFL}.

\textbf{Multi-step Forward Pass} --
Each communication round begins with a special type of forward pass at the \textit{weak} client.
First, the \textit{weak} client receives the input-side layers from the server as many as its memory space allows (step 1).
The exact number of layers can be pre-defined based on the client-side system capacity.
Then, it performs the forward pass through the received layers using the local dataset and records the output activation matrix.
The output matrix can be stored in either memory space or disk space.
Once the output is recorded, the weak client can discard the current layers and receive the next layers from the server (step 2).
Using the recorded matrices as input data, the client continues the forward pass.
The above steps are repeated until the weak client receives the last output-side layers from the server.

We name the forward pass described above \lq{}multi-step forward pass\rq{}.
Algorithm \ref{alg:MSFP} shows the pseudo-code of the multi-step forward pass.
At the end of the repeated forward passes, the \textit{weak} client ends up having the last set of output-side layers, and it is ready to run local training steps.
Note that the multi-step forward pass is performed only once per communication round.
After the activation matrices are recorded right before the last output-side sub-model, the recorded matrices are re-used multiple times during local training steps.
This design choice allows us to significantly reduce the memory footprint and the computational cost at the \textit{weak} client while introducing slight bias into the solution.

\textbf{Local Training on Weak Clients} -- 
Algorithm \ref{alg:EmbracingFL} shows the \texttt{EmbracingFL} framework from the \textit{weak} client's perspective.
Since the \textit{strong} clients perform the full local training and model aggregation like FedAvg, we only focus on the \textit{weak} client's local training.
Once the multi-step forward pass is finished, the \textit{weak} client performs the local training steps using the recorded intermediate activation matrices as the input data (step 3).
The parameter update rule of the proposed strategy can be formally defined as follows.
For simplicity, we consider a simple case where there are $m$ clients in total and $s \leq m$ clients are \textit{strong}.
\begin{align}
    \mathbf{x_t} &= \left( \mathbf{y_t}, \mathbf{z_t} \right) \nonumber \\
    \mathbf{y_{t+1}} &= \mathbf{y}_t - \frac{\eta}{s} \sum_{i=1}^{s} \sum_{j=0}^{\tau - 1} \nabla f(\mathbf{y}_{t,j}^i) \nonumber \\
    \mathbf{z_{t+1}} &= \mathbf{z}_t - \frac{\eta}{m} \sum_{i=1}^{m} \sum_{j=0}^{\tau - 1} \nabla f(\mathbf{z}_{t,j}^i), \nonumber
\end{align}
where $\mathbf{y}$ is the input-side layers trained by $s$ strong clients only and $\mathbf{z}$ is all the other output-side layers that are trained by all $m$ clients.
Note that, because the weak clients do not perform the backward pass at $\mathbf{y}$, $\nabla f(\mathbf{y}_t^i)$ are averaged across the $s$ strong clients only.
Finally, the locally trained layers are aggregated at the server and averaged to obtain the global model parameters (step 4).

\begin{algorithm}[t]
\caption{
    Multi-Step Forward Pass.
}
\label{alg:MSFP}
\begin{algorithmic}
    \STATE \textbf{Input:} $D^i$: local dataset
    \STATE $\mathbf{y}_{t,0}^i = \mathbf{y}_{t}$ \COMMENT{Receive $\mathbf{y}_{t,0}^{i}$ from the server.}
    \STATE $\bar{D}^i = \textrm{Feed-Forward}(\mathbf{y}_{t,0}^{i}, D^i)$  \COMMENT{Store the output of $\mathbf{y}_{t,0}^{i}$}
    \STATE $\mathbf{z}_{t,0}^i = \mathbf{z}_{t}$ \COMMENT{Discard $\mathbf{y}_{t,0}^i$ and receive $\mathbf{z}_{t,0}^{i}$.}
    \STATE \textbf{Output:} $\mathbf{z}_{t,0}^i, \bar{D}^i$
\end{algorithmic}
\end{algorithm}

\begin{algorithm}[t]
\caption{
    \texttt{EmbracingFL} (Weak Clients Training).
}
\label{alg:EmbracingFL}
\begin{algorithmic}
    \STATE{\textbf{Input:} $\tau$: the aggregation interval, $T$: the number of communication rounds}
    \STATE{$\mathbf{x}_0 \leftarrow$ the initial global model}
    \FOR{$t \in \{0, \cdots, T-1 \}$}
        \STATE{$\mathbf{z}_{t,0}^{i}, \bar{D}^i = \textrm{MultiStepForwardPass}(D^i)$}
        \FOR{$j \in \{0, \cdots, \tau - 1\}$}
            \STATE{$\xi_i \leftarrow$ a random data sampled from $\bar{D}^i$}
            \STATE{$\mathbf{z}_{t,j+1}^i = \textrm{LocalUpdate}(\mathbf{z}_{t,j}^i, \xi_i)$}
        \ENDFOR
        \STATE{$\mathbf{z}_{t+1} = \frac{1}{m} \sum_{i=1}^{m} \mathbf{z}_{t}^i$} \COMMENT{Model aggregation}
    \ENDFOR
    \STATE{\textbf{Output:} $\mathbf{x}_{T}$}
\end{algorithmic}
\end{algorithm}

The proposed partial training strategy has two critical benefits as follows.
First, all the available clients can effectively participate in the global model training regardless of their system resources.
The weak clients can focus only on the assigned output-side sub-model during the backward pass.
Thus, the memory footprint as well as the computational workload is proportionally reduced, and it enables the weak clients to join the federated optimization process.
As already shown in several previous works, the gradient computation complexity is proportional to the number of layers and the filter size \cite{lee2017parallel,asghar2019assessment}.
Because our method directly removes trainable layers for the \textit{weak} clients, their gradient computation cost is expected to be proportionally reduced.
Second, the proposed multi-step forward pass allows all the clients to view the same model architecture, and thus the locally trained network layers can be directly averaged across all the clients.
This property dramatically eases the implementation complexity.
This also allows us to theoretically analyze the convergence properties of the proposed strategy, providing a convergence guarantee.

\begin{figure*}[t]
\centering
\includegraphics[width=1.8\columnwidth]{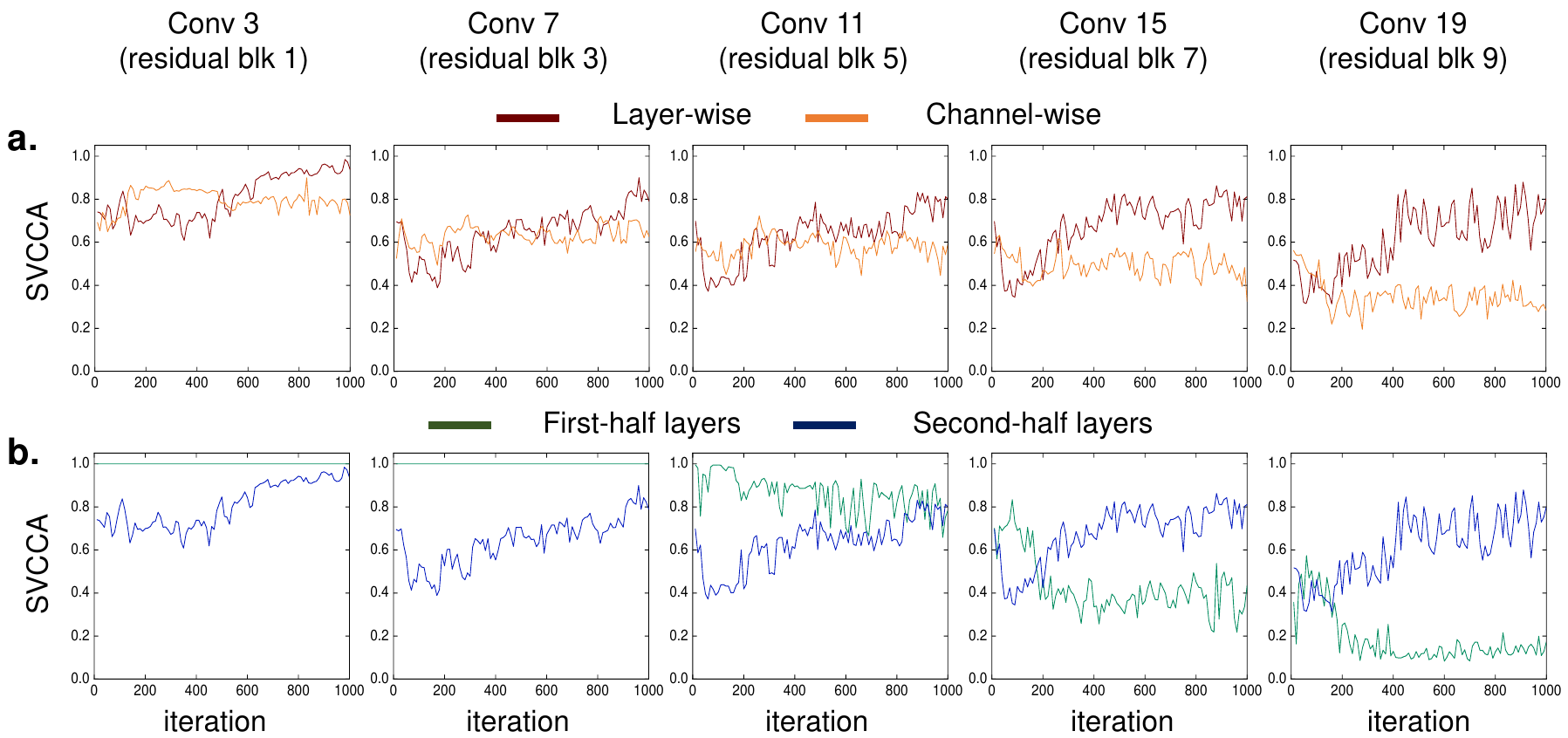}
\caption{
    The SVCCA comparison across different partial synchronization strategies.
    ResNet20 is trained on non-IID CIFAR-10 for 1000 iterations.
    The maximum SVCCA is measured among 128 local models.
    \textbf{a}. The SVCCA comparison between \textit{layer-wise} and \textit{channel-wise} approaches.
    \textbf{b}. The SVCCA comparison between \textit{input-side} layer-wise and \textit{output-side} layer-wise approaches.
}
\label{fig:svcca2}
\end{figure*}

\subsection {Ablation Study on Partial Model Synchronization} \label{sec:impact}
To better understand the impact of partial model synchronization on the model discrepancy across the clients, we analyze how the data representation similarity evolves during training.
We first compare \textit{layer-wise} and \textit{channel-wise} schemes.
The \textit{layer-wise} approach synchronizes the output-side half layers while the \textit{channel-wise} synchronizes half channels at every layer.
We train ResNet20 on non-IID CIFAR-10 using 128 clients and partially synchronize the model after every 10 local training steps.
The SVCCA is measured using the validation dataset after every model synchronization.
Figure \ref{fig:svcca2}.\textbf{a}. shows the similarity comparisons.
The \textit{layer-wise} approach shows a higher degree of similarity than the \textit{channel-wise} at all five layers.
Especially at the output-side layers (Conv 15 and Conv 19), there is a large gap between the two approaches.
This observation strongly supports our design choice of layer-wise partial synchronization in \texttt{EmbracingFL}.

We also present the data similarity comparison between two different \textit{layer-wise} schemes: the \textit{first-half} approach and the \textit{second-half} approach, as shown in Figure \ref{fig:svcca2}.\textbf{b}.
Note that the \textit{first-half} approach can be considered as \textit{InclusiveFL} proposed in \cite{liu2022no}.
The \textit{first-half} always shows the SVCCA values of $1$ at the input-side layers (Conv 3 and Conv 7) because the corresponding parameters are identical across all the clients.
However, the \textit{first-half} rapidly loses the similarity at the output-side layers as the training progresses while the \textit{second-half} approach maintains the similarity.
This empirical study demonstrates that the overall degree of model discrepancy heavily depends on how fast the output-side layers become to have different data representations.
In \texttt{EmbracingFL}, the output-side layers are trained by all the clients and globally synchronized.
Therefore, even though the \textit{weak} clients contribute to the output-side sub-model only, the overall model discrepancy is expected to be considerably reduced by the proposed layer-wise partial model training method.

\subsection {Convergence Analysis} \label{sec:theory}
Herein, we provide a theoretical analysis of the convergence properties of \texttt{EmbracingFL}.
We consider federated optimization problems as follows.
\begin{align}
    \underset{\mathbf{x} \in \mathbb{R}^d}{\min}\left[F(\mathbf{x}) := \frac{1}{m} \sum_{i=1}^{m} F_i(\mathbf{x}) \right] \label{eq:objective},
\end{align}
where $m$ is the number of local models and $F_i(\mathbf{x}) = \mathop{\mathbb{E}}_{\xi_i \sim D_i}[ F_i(\mathbf{x}, \xi_i)]$
is the local objective function associated with local data distribution $D_i$.
We define $\nabla_k F_i(\mathbf{x}, \xi_i)$ as the stochastic gradient with respect to the model partition $k$ and client $i$'s local data $\xi_i$. We omit $\xi_i$ for simplicity.

For simplicity, we consider the case where the clients are either \textit{strong} or \textit{weak}.
This analysis can be easily extended to more general cases with various client capacities.
It is also trivial to extend our analysis to the mini-batch version of FedAvg.
The analysis is centered around the following common assumptions.\\
\textbf{Assumption 1.} Each local objective function is $L_k$-smooth, $\forall k \in \{y, z\}$, that is, $\| \nabla_k F_i (\mathbf{x}) - \nabla_k F_i (\mathbf{x'}) \| \leq L_k \| \mathbf{x} - \mathbf{x'} \|, \forall i \in \{ 1, \cdots, m \}$ and $\forall k \in \{y, z\}$. \\
\textbf{Assumption 2.} The local gradient is an unbiased estimator of the local full-batch gradient for all model partitions: $\mathop{\mathbb{E}}_{\xi_i} \left[ \nabla_k f(\mathbf{x}_{t,j}^{i}) \right] = \nabla_k F_i(\mathbf{x}_{t,j}^{i}), \forall i \in \{ 1, \cdots, m \}$ and $\forall k \in \{y, z\}$.\\
\textbf{Assumption 3.} The gradient at each client has bounded variance: $\mathop{\mathbb{E}}_{\xi_i} \left[ \| \nabla_k f(\mathbf{x}_{t,j}^{i}) - \nabla_k F_i (\mathbf{x}_{t,j}^{i}) \|^2 \right] \leq \sigma_k^2, \forall i \in \{ 1, \cdots, m \}$ and $\forall k \in \{y, z\}$.
Likewise, the full batch gradient has bounded variance: $\mathop{\mathbb{E}} \left[ \| \nabla_k F_i(\mathbf{x}_{t,j}^{i}) - \nabla_k F (\mathbf{x}_{t,j}^{i}) \|^2 \right] \leq \bar{\sigma}_k^2, \forall i \in \{ 1, \cdots, m \}$ and $\forall k \in \{y, z\}$.

\begin{table*}[ht!]
\footnotesize
\centering
\caption{
    The model size of the three different types of clients.
}
\label{tab:size}
\begin{tabular}{llll} \toprule
Removed layers of Resnet20 (CIFAR-10) & Number of parameters ($p$) & Number of activations ($a$) & Capacity \\ \midrule
(Strong) - & 272,762 & 6,947,136 & 1.00 \\ 
(Moderate) The first conv. layer + the first 3 residual blocks & 257,994 & 2,752,832 & 0.42 \\
(Weak) The first conv. layer + the first 6 residual blocks & 206,346 & 917,824 & 0.16 \\ \midrule
Removed layers of CNN (FEMNIST) & Number of parameters ($p$) & Number of activations ($a$) & Capacity \\ \midrule
(Strong) - & 6,603,710 & 39,742 & 1.00 \\ 
(Moderate) The first 2 conv. layers & 6,551,614 & 2,110 & 0.99 \\ 
(Weak) The first 2 conv. layers + one dense layer & 127,038 & 62 & 0.02 \\ \midrule
Removed layers of Bidirectional LSTM (IMDB) & Number of parameters ($p$) & Number of activations ($a$) & Capacity \\ \midrule
(Strong) - & 3,611,137 & 1,310,720 & 1.00 \\ 
(Moderate) The embedding layer & 1,051,137 & 1,310,720 & 0.48 \\ 
(Weak) The embedding layer and the first half of the words (128) & 1,051,137 & 655,360 & 0.35 \\ \bottomrule
\end{tabular}
\end{table*}

Then, our analysis of the convergence rate is as follows.
The proof is in the Appendix \ref{app:proof}.
\begin{theorem}
\label{theorem:main}
Suppose all $m$ local models are initialized to the same point $\mathbf{x}_0$. Under Assumption $1 \sim 3$, if Algorithm \ref{alg:EmbracingFL} runs for $T$ communication rounds and $\eta \leq \min{ \left\{ \frac{1}{\tau L_{max}}, \frac{1}{4 L_{max} \sqrt{\tau (\tau - 1)}} \right\} }$, the average-squared gradient norm of $\mathbf{x}_t$ is bounded as follows.
\begin{align}
    \frac{1}{T} \sum_{t=0}^{T-1} \mathop{\mathbb{E}}\left[ \left\| \nabla F(\mathbf{x}_t) \right\|^2 \right] & \leq \frac{14}{3 T \eta \tau} \left( F(\mathbf{x}_0) - F(\mathbf{x}_{*}) \right) \nonumber \\
    & + \left( \frac{7 L_y \eta}{3s} + \frac{16 L_y^2 \eta^2 (\tau - 1)}{3} \right) \sigma_y^2  \nonumber \\
    & + \left( \frac{14}{3s} + \frac{64 L_y^2 \eta^2 \tau (\tau - 1)}{3} \right) \bar{\sigma}_y^2 \nonumber \\
    & + \left( \frac{7L_z \eta}{3m} + \frac{8 L_z^2 \eta^2 (\tau - 1)}{3} \right) \sigma_z^2 \nonumber \\
    & + \left( \frac{32 L_z^2 \eta^2 \tau (\tau - 1)}{3} \right) \bar{\sigma}_z^2
\end{align}
\end{theorem}

\textbf{Remark 1.}
Our analysis shows that the average magnitude of the gradient is bounded by a constant.
Given a fixed learning rate, the right-hand side becomes a constant error bound at convergence (when $T \rightarrow \infty$).
We first see that \texttt{EmbracingFL} converges regardless of how many \textit{weak} clients participate in the training.
The convergence is also guaranteed regardless of how many layers are assigned to the \textit{weak} clients.
This guarantee is a critical benefit of \texttt{EmbracingFL}, which enables weak clients to flexibly choose how many layers to train.
Cho et al. have shown that convergence is guaranteed for strongly convex problems even if the whole model is biased \cite{cho2020client}.

\textbf{Remark 2.}
The heterogeneous partial model training method causes a sub-linear speedup.
Even with a diminishing learning rate $\eta = \sqrt{\frac{m}{T}}$, the $\frac{14}{3s}\bar{\sigma}_y$ term on the right-hand side ends up dominating all the other terms as $T$ increases.
This implies that the model converges to a certain region close to a stationary point rather than the exact point.
The more the \textit{strong} clients, the closer the model approaches to the exact minimum.
When the number of \textit{strong} clients increases, the fraction of $\mathbf{y}$ decreases, and it ends up making the model consist only of $\mathbf{z}$.
In that case, $\sigma_y$ and $\bar{\sigma}_y$ become $0$.
Under this homogeneous setting, if $\eta = \sqrt{\frac{m}{T}}$, the complexity becomes $O(\frac{1}{\sqrt{mT}}) + O(\frac{m}{T})$.
Thus, if $T > m^3$, it achieves linear speedup.
This result is consistent with the FedAvg analysis shown in many previous works \cite{wang2021cooperative,wang2020tackling,asad2020fedopt,yang2021achieving}.

\section {Experiments} \label{sec:setting}

\textbf{Datasets, Models, and Hyper-Parameters} --
We evaluate the performance of \texttt{EmbracingFL} using three popular benchmarks, CIFAR-10 (ResNet20), Federated MNIST (FEMNIST) \cite{caldas2018leaf} (CNN), and IMDB review (LSTM) (See Appendix for detailed settings).
The training runs for 10,000, 2,000, and 4,000 local steps, respectively.
In all experiments, the local batch size is 32 and the number of local steps per communication round ($\tau$) is 10.

\textbf{Non-IID Data Settings} --
We generate non-IID versions of CIFAR-10 and IMDB using Dirichlet distribution ($\alpha = 0.1$).
The data samples of each label are distributed based on a specific Dirichlet distribution such that each local dataset has unbalanced labels.
Consequently, each local dataset contains a different number of samples.
Given 3,500 writers' samples in FEMNIST, we assign two random writers' samples to each client.
We report accuracy averaged across at least three runs.

\textbf{Client Capacity Settings} -- 
We define three client types, \textit{strong}, \textit{moderate}, and \textit{weak}.
Table \ref{tab:size} shows their model sizes.
We quantify the memory footprint of a \textit{weak} client as $2p_w + 2a_w$, where $p_w$ and $a_w$ are the numbers of model parameters and that of activations (errors) in the \textit{weak} model, respectively.
Then, we define \textit{Capacity} of the \textit{weak} client as $C_w = (2p_w + 2a_w) / (2p + 2a)$, where $p$ and $a$ are the number of parameters and that of activations in the full model, respectively.
This ratio shows how much of the memory space is used by the \textit{weak} client as compared to the \textit{strong} client.
The \textit{Capacity} column in Table \ref{tab:size} shows such ratios.

Note that the \textit{moderate} clients in FEMNIST experiment have an almost similar capacity to the \textit{strong} clients.
The convolutional neural network designed for FEMNIST classification consists only of four layers, and the first fully-connected layer takes up more than $90\%$ of the total model parameters.
Thus, there is no possible way of splitting such a shallow model to have a capacity of $0.4$.
We drop the first convolution layer for the \textit{moderate} clients but its memory consumption is still $99\%$ of that of the \textit{strong} clients' model.
The \textit{weak} client's model does not contain the first fully-connected layer and thus it only has a capacity of $0.02$.
With this extremely unbalanced model-splitting, we can generate more realistic heterogeneous FL environments.

\begin{figure*}[t]
\centering
\includegraphics[width=2\columnwidth]{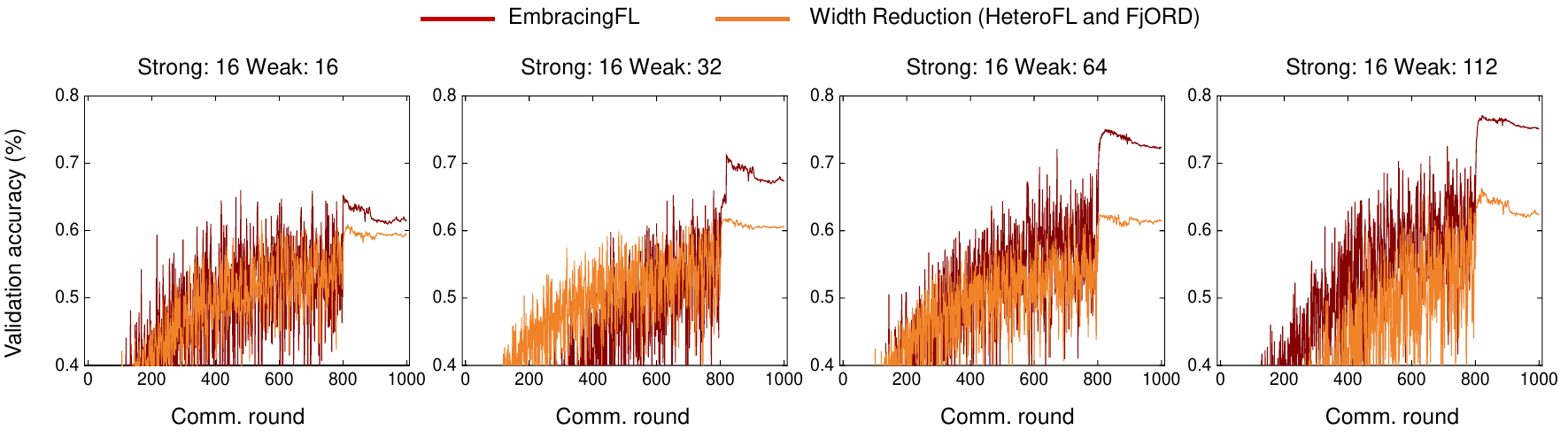}
\caption{
    The comparison of CIFAR-10 validation accuracy curves (\texttt{EmbracingFL} vs. Width Reduction). As the number of \textit{weak} clients increases, the accuracy gap between the two methods becomes more significant.
}
\label{fig:cifar10}
\end{figure*}

\subsection {Heterogeneous Scaling with Weak Clients}
We compare \texttt{EmbracingFL} with the state-of-the-art partial model training methods, HeteroFL \cite{diao2020heterofl} and FjORD \cite{horvath2021fjord}.
Essentially, HeteroFL and FjORD can be categorized into the same scheme which reduces the width of each network layer.
We call them \lq{}Width Reduction\rq{} method for short.
Note that the dynamic version of HeteroFL and the self knowledge distillation in FjORD are not directly related to the partial model training scheme.
That is, those additional features could be independently applied to \texttt{EmbracingFL} as well.
To focus only on the efficacy of different partial model training schemes, therefore, we only consider the static version of HeteroFL and the `ordered dropout' feature in FjORD.
We scale up CIFAR-10 (ResNet20) training by fixing the number of \textit{strong} clients to 16 and adding only \textit{weak} clients.
When using the width reduction method, the \textit{weak} clients drop $80\%$ of the channels at every layer (See Appendix \ref{sec:additional}).
This setting yields a similar capacity ($C \approx 0.18$) as the \textit{weak} clients in \texttt{EmbracingFL} ($C \approx 0.16$).
In this experiment, we enforce all 128 clients to always join the training.

\begin{table}[t]
\footnotesize
\centering
\caption{
    The CIFAR-10 scaling performance comparison between \texttt{EmbracingFL} and width reduction methods (HeteroFL \cite{diao2020heterofl} and FjORD \cite{horvath2021fjord}).
    Given a fixed iteration budget (10,000), as the number of \textit{weak} clients increases, \texttt{EmbracingFL} achieves consistently higher accuracy than Width Reduction.
}
\label{tab:scaling}
\begin{tabular}{rrcc} \toprule
\# of strong & \# of weak & Width Reduction & \texttt{EmbracingFL} \\ \midrule
16 & 0 & \multicolumn{2}{c}{$60.15\pm 1.5\%$} \\
16 & 16 & $61.34\pm 2.1\%$ & $\textbf{66.62}\pm 1.1\%$ \\
16 & 32 & $62.09\pm 1.5\%$ & $\textbf{72.60}\pm 1.2\%$ \\
16 & 64 & $63.68\pm 3.3\%$ & $\textbf{74.79}\pm 0.8\%$ \\
16 & 112 & $65.01\pm 2.9\%$ & $\textbf{77.34}\pm 1.6\%$ \\ \bottomrule
\end{tabular}
\end{table}

\begin{table*}[t]
\footnotesize
\centering
\caption{
    The non-IID CIFAR-10 classification performance under various heterogeneous FL settings.
}
\label{tab:cifar10_short}
\begin{tabular}{lrlrlrlcc} \toprule
& \multicolumn{2}{c}{Strong client} & \multicolumn{2}{c}{Moderate client} & \multicolumn{2}{c}{Weak client} & Avg. Capacity & Validation accuracy \\ \midrule
case 1 & 128 & $(100\%)$ & 0 & $(0\%)$ & 0 & $(0\%)$ & 1.00 & $80.45 \pm 0.2\%$ \\ \midrule
case 2 & 64 & $(50\%)$ & 64 & $(50\%)$ & 0 & $(0\%)$ & 0.71 & $80.20 \pm 0.2\%$ \\ 
case 3 & 32 & $(25\%)$ & 96 & $(75\%)$ & 0 & $(0\%)$ & 0.57 & $79.78 \pm 0.4\%$ \\
case 4 & 16 & $(12.5\%)$ & 112 & $(87.5\%)$ & 0 & $(0\%)$ & 0.49 & $79.72 \pm 0.6\%$\\ \midrule
case 5 & 64 & $(50\%)$ & 0 & $(0\%)$ & 64 & $(50\%)$ & 0.58 & $80.08 \pm 0.4\%$ \\ 
case 6 & 32 & $(25\%)$ & 0 & $(0\%)$ & 96 & $(75\%)$ & 0.37 & $78.91 \pm 0.5\%$ \\
case 7 & 16 & $(12.5\%)$ & 0 & $(0\%)$ & 112 & $(87.5\%)$ & 0.27 & $76.98 \pm 1.6\%$\\ \midrule
case 8 & 32 & $(25\%)$ & 32 & $(25\%)$ & 64 & $(50\%)$ & 0.44 & $80.33 \pm 0.6\%$ \\ 
case 9 & 16 & $(12.5\%)$ & 32 & $(25\%)$ & 80 & $(62.5\%)$ & 0.33 & $79.63 \pm 0.1\%$ \\
case 10 & 16 & $(12.5\%)$ & 16 & $(12.5\%)$ & 96 & $(75\%)$ & 0.30 & $77.63 \pm 0.2\%$ \\ \bottomrule
\end{tabular}
\end{table*}

\begin{table*}[t]
\footnotesize
\centering
\caption{
    The non-IID FEMNIST classification performance under various heterogeneous FL settings.
}
\label{tab:femnist_short}
\begin{tabular}{lrlrlrlcc} \toprule
& \multicolumn{2}{c}{Strong client} & \multicolumn{2}{c}{Moderate client} & \multicolumn{2}{c}{Weak client} & Avg. Capacity & Validation accuracy \\ \midrule
case 1 & 128 & $(100\%)$ & 0 & $(0\%)$ & 0 & $(0\%)$ & 1.00 & $81.27 \pm 0.6\%$ \\ \midrule
case 2 & 64 & $(50\%)$ & 64 & $(50\%)$ & 0 & $(0\%)$ & 0.99 & $81.80 \pm 0.5\%$ \\ 
case 3 & 32 & $(25\%)$ & 96 & $(75\%)$ & 0 & $(0\%)$ & 0.99 & $81.55 \pm 0.5\%$ \\
case 4 & 16 & $(12.5\%)$ & 112 & $(87.5\%)$ & 0 & $(0\%)$ & 0.99 & $80.44 \pm 0.7\%$\\ \midrule
case 5 & 64 & $(50\%)$ & 0 & $(0\%)$ & 64 & $(50\%)$ & 0.51 & $80.80 \pm 0.7\%$ \\ 
case 6 & 32 & $(25\%)$ & 0 & $(0\%)$ & 96 & $(75\%)$ & 0.27 & $80.78 \pm 0.5\%$ \\ 
case 7 & 16 & $(12.5\%)$ & 0 & $(0\%)$ & 112 & $(87.5\%)$ & 0.14 & $79.63 \pm 0.3\%$\\ \midrule
case 8 & 32 & $(25\%)$ & 32 & $(25\%)$ & 64 & $(50\%)$ & 0.51 & $81.29 \pm 0.4\%$ \\ 
case 9 & 16 & $(12.5\%)$ & 32 & $(25\%)$ & 80 & $(62.5\%)$ & 0.39 & $81.18 \pm 0.3\%$ \\
case 10 & 16 & $(12.5\%)$ & 16 & $(12.5\%)$ & 96 & $(75\%)$ & 0.26 & $81.17 \pm 0.4\%$ \\ \bottomrule
\end{tabular}
\end{table*}

\begin{table*}[ht!]
\footnotesize
\centering
\caption{
    The non-IID IMDB classification performance under various heterogeneous FL settings.
}
\label{tab:imdb_short}
\begin{tabular}{lrlrlrlcc} \toprule
& \multicolumn{2}{c}{Strong client} & \multicolumn{2}{c}{Moderate client} & \multicolumn{2}{c}{Weak client} & Avg. Capacity & Validation accuracy \\ \midrule
case 1 & 128 & $(100\%)$ & 0 & $(0\%)$ & 0 & $(0\%)$ & 1.00 & $82.15 \pm 0.3\%$ \\ \midrule
case 2 & 64 & $(50\%)$ & 64 & $(50\%)$ & 0 & $(0\%)$ & 0.74 & $82.21 \pm 0.3\%$ \\ 
case 3 & 32 & $(25\%)$ & 96 & $(75\%)$ & 0 & $(0\%)$ & 0.61 & $82.09 \pm 0.5\%$ \\
case 4 & 16 & $(12.5\%)$ & 112 & $(87.5\%)$ & 0 & $(0\%)$ & 0.55 & $82.55 \pm 0.3\%$\\ \midrule
case 5 & 64 & $(50\%)$ & 0 & $(0\%)$ & 64 & $(50\%)$ & 0.68 & $80.38 \pm 0.2\%$ \\ 
case 6 & 32 & $(25\%)$ & 0 & $(0\%)$ & 96 & $(75\%)$ & 0.51 & $77.09 \pm 0.8\%$ \\
case 7 & 16 & $(12.5\%)$ & 0 & $(0\%)$ & 112 & $(87.5\%)$ & 0.43 & $74.89 \pm 0.4\%$\\ \midrule
case 8 & 32 & $(25\%)$ & 32 & $(25\%)$ & 64 & $(50\%)$ & 0.55 & $79.09 \pm 0.4\%$ \\ 
case 9 & 16 & $(12.5\%)$ & 32 & $(25\%)$ & 80 & $(62.5\%)$ & 0.46 & $79.41 \pm 0.3\%$ \\
case 10 & 16 & $(12.5\%)$ & 16 & $(12.5\%)$ & 96 & $(75\%)$ & 0.45 & $77.32 \pm 0.4\%$ \\ \bottomrule
\end{tabular}
\end{table*}

\begin{table*}[t]
\footnotesize
\centering
\caption{
    The non-IID CIFAR-10 classification performance under various heterogeneous FL settings.
    Width Reduction corresponds to HeteroFL \cite{diao2020heterofl} and FjORD \cite{horvath2021fjord}.
}
\label{tab:comparison}
\begin{tabular}{lrlrlrlccc} \toprule
& \multicolumn{2}{c}{Strong client} & \multicolumn{2}{c}{Moderate client} & \multicolumn{2}{c}{Weak client} & Avg. Capacity & Width Reduction & \texttt{EmbracingFL} \\ \midrule
case 1 & 128 & $(100\%)$ & 0 & $(0\%)$ & 0 & $(0\%)$ & 1.00 & \multicolumn{2}{c}{$80.45 \pm 0.2\%$} \\ \midrule
case 2 & 64 & $(50\%)$ & 64 & $(50\%)$ & 0 & $(0\%)$ & 0.71 & $76.77 \pm 1.3\%$ & $\mathbf{80.20}\pm 0.2\%$ \\
case 3 & 32 & $(25\%)$ & 96 & $(75\%)$ & 0 & $(0\%)$ & 0.57 & $67.92\pm 2.1\%$ & $\mathbf{79.78}\pm 0.4\%$ \\
case 4 & 16 & $(12.5\%)$ & 112 & $(87.5\%)$ & 0 & $(0\%)$ & 0.49 & $59.03\pm 0.8\%$ 
 & $\mathbf{79.72}\pm 0.6\%$ \\ \midrule
case 5 & 64 & $(50\%)$ & 0 & $(0\%)$ & 64 & $(50\%)$ & 0.58 & $72.97 \pm 2.5\%$ & $\mathbf{80.08}\pm 0.4\%$ \\
case 6 & 32 & $(25\%)$ & 0 & $(0\%)$ & 96 & $(75\%)$ & 0.37 & $69.70 \pm 2.1 \%$ & $\mathbf{78.91}\pm 0.5\%$ \\
case 7 & 16 & $(12.5\%)$ & 0 & $(0\%)$ & 112 & $(87.5\%)$ & 0.27 & $54.53 \pm 2.9\%$ & $\mathbf{76.98}\pm 1.6\%$ \\ \midrule
case 8 & 32 & $(25\%)$ & 32 & $(25\%)$ & 64 & $(50\%)$ & 0.44 & $67.93\pm 1.5\%$ & $\mathbf{80.33}\pm 0.6\%$ \\
case 9 & 16 & $(12.5\%)$ & 32 & $(25\%)$ & 80 & $(62.5\%)$ & 0.33 & $59.59\pm 0.8\%$ & $\mathbf{79.63}\pm 0.1\%$ \\
case 10 & 16 & $(12.5\%)$ & 16 & $(12.5\%)$ & 96 & $(75\%)$ & 0.30 & $58.17\pm 2.2\%$ & $\mathbf{77.63}\pm 0.2\%$ \\ \bottomrule
\end{tabular}
\end{table*}

\begin{table*}[t]
\footnotesize
\centering
\caption{
    The number of comm. rounds to achieve the target CIFAR-10 accuracy.
    This comparison demonstrates how effective \textit{EmbracingFL} is in terms of computational and communication costs.
}
\begin{tabular}{cclclclccc} \toprule
&  \multicolumn{2}{c}{Strong client} & \multicolumn{2}{c}{Moderate client} & \multicolumn{2}{c}{Weak client} & Target Acc. & Width Reduction & \texttt{EmbracingFL} \\ \midrule
case 1 & 128 & ($100\%$) & 0 & ($0\%$) & 0 & ($0\%$) & $80\%$ & \multicolumn{2}{c}{890} \\ \midrule
case 2 & 64 & ($50\%$) & 64 & ($50\%$) & 0 & ($0\%$) & $76\%$ & 867 & \textbf{817} \\
case 3 & 32 & ($25\%$) & 96 & ($75\%$) & 0 & ($0\%$) & $67\%$ & 982 & \textbf{801} \\
case 4 & 16 & ($12.5\%$) & 112 & ($87.5\%$) & 0 & ($0\%$) & $59\%$ & 836 & \textbf{483} \\ \midrule
case 5 & 64 & ($50\%$) & 0 & ($0\%$) & 64 & ($50\%$) & $72\%$ & 813 & \textbf{811} \\
case 6 & 32 & ($25\%$) & 0 & ($0\%$) & 96 & ($75\%$) & $69\%$ & 806 & \textbf{629} \\
case 7 & 16 & ($12.5\%$) & 0 & ($0\%$) & 112 & ($87.5\%$) & $54\%$ & 939 & \textbf{572} \\ \midrule
\end{tabular}
\label{tab:rounds}
\end{table*}

\begin{table*}[ht!]
\scriptsize
\centering
\caption{
    Timing breakdown comparison (ResNet20). The timings are measured on a real edge device, OnePlus 9 Pro.
    Three model sizes are considered as shown in Table \ref{tab:size}.
    We measure the time for processing a single batch of size 32.
    The timings are averaged across 10 iterations.
}
\begin{tabular}{cc|rrr|rrr} \toprule
\multicolumn{2}{c|}{} & \multicolumn{3}{c|}{\texttt{EmbracingFL}} & \multicolumn{3}{c}{Width Reduction} \\ \midrule
Workload & Model Size & Computation & I/O & End-to-End & Computation & I/O & End-to-End \\ \midrule
 \multirow{3}{*}{Feed-forward} & Strong & \multirow{3}{*}{2095.4 ms} & - & 2095.4 ms & 2010.7 ms & - & 2010.7 ms \\ 
& Moderate & & 678.4 ms & 2773.8 ms & 1431.0 ms & - & 1431.0 ms \\
& Weak & & 1316.8 ms & 3412.2 ms & 936.7 ms & - & 936.7 ms \\ \midrule
\multirow{3}{*}{Backpropagation} & Strong & 419,643.8 ms & \multirow{3}{*}{-} & 419,643.8 ms & 421,047.5 ms & \multirow{3}{*}{-} & 421,047.5 ms \\
& Moderate & 197,265.0 ms & & 197,265.0 ms & 317,669.7 ms & & 317,669.7 ms \\
& Weak & 85,448.3 ms & & 85,448.3 ms & 187,580.4 ms & & 187,580.4 ms \\ \bottomrule
\end{tabular}
\label{tab:timing}
\end{table*}

Figure \ref{fig:cifar10} shows CIFAR-10 learning curve comparisons and Table \ref{tab:scaling} shows the corresponding accuracy comparisons.
Because each client has its own local training data, as having more \textit{weak} clients, the global model is expected to achieve higher accuracy.
While both approaches obtain a higher accuracy as more \textit{weak} clients join the training, there is a huge accuracy gap between \texttt{EmbracingFL} and the width reduction method.
This result demonstrates that \texttt{EmbracingFL} more effectively enables the \textit{weak} client participation.

\subsection {Classification Performance Under Various Heterogeneous FL Environments}
To evaluate the performance of \texttt{EmbracingFL} under realistic FL environments, we measure the validation accuracy under a variety of heterogeneous client capacity settings.
Table \ref{tab:cifar10_short}, \ref{tab:femnist_short}, and \ref{tab:imdb_short} show the classification performance of CIFAR-10, FEMNIST, and IMDB, respectively.
We consider ten different client capacity settings.
The \textit{case 1} is the baseline in which all $128$ clients are all \textit{strong} clients.
The \textit{Avg. Capacity} column is calculated by $C = C_s R_s + C_m R_m + C_w R_w$, where $C_s$, $C_m$, and $C_w$ are the capacity of each client type shown in Table \ref{tab:size} and $R_s$, $R_m$, and $R_w$ are the ratios of each client type to the total 128 clients.
In these experiments, we activate $25\%$ clients randomly sampled from 128 clients at every communication round.

Overall, \texttt{EmbracingFL} effectively utilizes the \textit{moderate} and \textit{weak} clients and achieves similar accuracy to the \textit{strong} client-only accuracy.
When some clients are \texttt{moderate}, we do not see any significant accuracy drop in all three benchmarks ($< 0.8\%$).
Under extreme settings (case 4, 7, 9, and 10) where the ratio of \textit{strong} clients is only $12.5\%$, \texttt{EmbracingFL} still achieves good accuracy comparable to case 1.
These results show that \texttt{EmbracingFL} maximizes the weak clients' contribution to the global model training and minimizes the accuracy drop.

\subsection {Comparative Study}

Table \ref{tab:comparison} shows CIFAR-10 classification performance comparison between \texttt{EmbracingFL} and the SOTA heterogeneous FL methods.
See Appendix \ref{sec:additional} for the detailed width reduction settings.
\texttt{EmbracingFL} achieves higher accuracy than the width reduction method under all the nine heterogeneous FL settings.
This comparison demonstrates that \texttt{EmbracingFL} enables the \textit{weak} client participation more effectively than the SOTA width reduction method.

Table \ref{tab:rounds} shows the number of communication rounds to achieve the target accuracy in CIFAR-10 experiments.
While Table \ref{tab:comparison} shows the best accuracy achieved within a certain fixed number of rounds, this comparison directly shows how fast each method can reach a certain target accuracy.
As expected, \texttt{EmbracingFL} achieves the target accuracy in much fewer communication rounds than the width reduction methods.
This result clearly shows the benefits of \texttt{EmbracingFL} with respect to the system efficiency.

\begin{table}[t]
\footnotesize
\centering
\caption{
    The CIFAR-10 (ResNet20) accuracy achieved with different batch normalization settings.
    The experiments are performed with 16 \textit{strong} clients and 112 \textit{weak} clients.
}
\label{tab:bn}
\begin{tabular}{ccc} \toprule
FL framework & Batch normalization & Accuracy \\ \midrule
Width Reduction & Static BN & $64.01 \pm 2.9\%$\\ 
Width Reduction & Global BN & $19.21\pm 1.5\%$\\ 
\texttt{EmbracingFL} & Static BN & $74.49\pm 1.1\%$ \\ 
\texttt{EmbracingFL} & Global BN & $\mathbf{77.34}\pm 1.6\%$\\ \bottomrule
\end{tabular}
\end{table}

\subsection {Performance Analysis On Edge Device} \label{sec:edge}
To compare the system efficiency, we analyze the timing breakdowns of the partial model training methods, measured on a real edge device.
We implemented ResNet20 training code using MNN \cite{jiang2020mnn} software library and measured timings on OnePlus 9 Pro.
Table \ref{tab:timing} shows the timings for a single mini-batch of size 32.
Although Alg. \ref{alg:EmbracingFL} runs the multi-step forward pass using the entire local dataset at once, we measure the I/O time for a single mini-batch in this experiment to show the effective extra cost per iteration.
The reported timings are averaged across at least 10 measures.

First, while \texttt{EmbracingFL} has the same computation time of forward pass regardless of the model size, the width reduction method has a proportionally reduced computation time.
In addition, \texttt{EmbracingFL} even has an extra I/O cost that is caused by the multi-step forward pass shown in \ref{alg:MSFP}.
Since the \textit{weak} clients receive the input-side layers from the server twice, the I/O time is almost double the I/O time of \textit{moderate} clients.
Consequently, \texttt{EmbracingFL} has a relatively longer forward pass time than the width reduction method at the \textit{moderate} and \textit{weak} clients.

In contrast to the forward pass, \texttt{EmbracingFL} shows a remarkably shorter backward pass time as compared to the width reduction method.
This difference stems from the different number of activations and weight parameters handled by the two heterogeneous FL strategies.
Specifically, the width reduction method tends to more drastically reduce the number of parameters than our proposed method while having more output activations at each layer.
While it depends on the model architecture, the number of activations is usually an order of magnitude larger than that of the model parameters in modern artificial neural networks.
Therefore, we can generally expect \texttt{EmbracingFL} to have a shorter backward pass time than the width reduction method.
Considering that the backward pass time is dominant over the forward pass time, we conclude that \texttt{EmbracingFL} well reduces the workload of the \textit{moderate} and \textit{weak} clients, allowing practical scaling on heterogeneous systems.

While the total transferred data size is not affected, the multi-step forward pass slightly increases the communication cost due to the extra latency cost (multiple model transfers).
The latency cost on edge devices is usually measured as 10 $\sim$ 20 ms \cite{gadban2021analyzing,pelle2022cost}, and cloud services like Amazon S3 have a latency cost of 100 $\sim$ 200 ms \cite{guide2008amazon}.
Therefore, compared to the computational cost shown in Table \ref{tab:timing}, the extra latency cost is negligible.
Nevertheless, we consider understanding the impact of having less frequent multi-step forward passes as interesting future work.

\subsection {Global vs. Static Batch Normalization}
We find that the width reduction methods significantly lose accuracy when the Batch Normalization (BN) statistics are globally synchronized.
HeteroFL employs \lq{}static BN\rq{} that does not track the global moving statistics of the data.
Table \ref{tab:bn} shows the CIFAR-10 accuracy comparison across different combinations of heterogeneous FL frameworks and BN settings.
The width reduction dramatically loses accuracy when the BN statistics are synchronized globally (global BN).
Because the \textit{weak} clients and the \textit{strong} clients view the layers with a different width, the directly averaged BN statistics do not well represent the global dataset's characteristics.
When the static BN is applied to \texttt{EmbracingFL}, it rather harms the accuracy.
Since \texttt{EmbracingFL} makes all clients view the same model architecture, the averaged statistics effectively represent the global dataset.
Therefore, we can conclude that \texttt{EmbracingFL} is more resilient to inaccurate BN statistics.

\section {Conclusion} \label{sec:conclusion}
Enabling weak client participation is essential in realistic FL environments.
We proposed \texttt{EmbracingFL}, a general FL framework that tackles the system heterogeneity issue through a layer-wise partial training strategy.
Our analysis provides a convergence guarantee of the proposed algorithm, and the extensive empirical study proves its efficacy in realistic heterogeneous FL environments.
The proposed layer-wise partial training strategy is readily applicable to any existing FL applications.
We believe our work sheds light on a novel way of implementing practical large-scale FL applications that exploit all available clients regardless of their resource capacities.
Reducing the frequency of biased forward pass and the overall latency cost is an important future work.


\section {Acknowledgment}
This material is based upon work supported by ARO award W911NF1810400 and CCF-1763673, ONR Award No. N00014-16-1-2189.
This work was partly supported by Institute of Information \& communications Technology Planning \& Evaluation (IITP) grant funded by the Korea government(MSIT) (No.RS-2022-00155915, Artificial Intelligence Convergence Innovation Human Resources Development (Inha University)).
This work was supported by the National Research Foundation of Korea(NRF) grant funded by the Korea government(MSIT) (No. RS-2023-00279003).
This work was supported by INHA UNIVERSITY Research Grant.

\normalsize
\bibliography{mybib}

\onecolumn
\subsection {Detailed Experimental Settings} \label{appendix:exp}

\textbf{Software Settings} --
We use TensorFlow 2.10.4 to implement FL software framework which run on a GPU cluster.
The model aggregation is implemented using MPI.
We commonly use the following hyper-parameter settings for all three benchmarks.
The mini-batch SGD with momentum (0.9) is used as a local optimizer.
The local mini-batch size is 32.
The weight decay factor is $0.0001$.
The number of local steps is fixed at 10.

\textbf{System Settings} --
We run federated learning simulations on a GPU-based cluster.
Each process runs on a single GPU and trains multiple local models sequentially.
Then, after all the active local models are trained, the whole models are averaged using inter-process communications.
We used 8 compute nodes in total, each of which has 2 NVIDIA V100 GPUs.
Therefore, we do not report the end-to-end execution time.

\subsection {Model Architecture and Learning Rate Schedule} \label{sec:model}
For CIFAR-10 classification, we used the standard ResNet20 model proposed in \cite{he2016deep}.
The learning rate is initially set to $0.4$ and decayed by a factor of $10$ after 800 and 900 communication rounds twice.
We find that this late learning rate decays yield higher accuracy than the typical settings (twice decays after $50\%$ and $75\%$ of iteration budget).

For FEMNIST classification, we use a CNN model presented in \cite{caldas2018leaf}.
The network consists of two convolution layers and two dense layers.
The convolution filter size is $5 \times 5$ and the number of channels is 32 and 64 for the two layers.
The first dense layer size is 2048.
Each convolution layer is followed by a max-pooling layer with a filter size of $2 \times 2$.
The learning rate is set to $0.04$ and we do not decay it for the whole training.

For IMDB review sentiment analysis, we use a bidirectional LSTM.
The first layer is an embedding layer followed by a dropout (0.3).
The input dimension is 10,000, the output dimension is 256, and the maximum length of the input sentence is 256.
Then, the bidirectional LSTM layer is attached to the embedding layer.
The LSTM size is 256 and the output of each recurrence is followed by dropout (0.3).
Finally, the LSTM module is followed by a dense layer.
The learning rate is set to $0.1$ and decayed by a factor of $10$ after $200$ and $300$ communication rounds.

\subsection {SVCCA setting} \label{sec:svcca}
SVCCA \cite{raghu2017svcca} quantifies the data representation similarity between two matrices.
We use this useful technique to quantify the data representation similarity between two different local models.
We first collect the output activation matrices at the target layer using the validation dataset.
Given the two matrices, then we perform Singular Vector Decomposition (SVD) on them and obtain the top $4$ singular vectors from each.
Note that we chose to use the 4 singular vectors only because using more vectors does not strongly affect the result while significantly increasing the analysis time.
Finally, we run CCA using these two sets of singular vectors.
All the reported SVCCA is the average CCA coefficients.
We used the open-source SVCCA code \footnote{https://github.com/google/svcca}.
Because SVD requires extremely large memory space for large matrices, we processed several small validation batches and averaged the SVCCA values.

\subsection {Width Reduction Settings} \label{sec:additional}
We mainly compare \texttt{EmbracingFL} with the width reduction method using CIFAR-10 (ResNet20) benchmark.
Table \ref{tab:size2} shows the \textit{Capacity} of the local models for the width reduction method.
For \textit{moderate} clients, we drop $55\%$ of channels at every layer.
For \textit{weak} clients, we drop $80\%$ of channels at every layer.
For $p$, we only count the weight parameters.
Under this setting, the width reduction method yields a similar memory footprint as that of \texttt{EmbracingFL} shown in Table \ref{tab:size}.
Note that in these experiments, all 128 clients always join the training.

\begin{table}[ht!]
\scriptsize
\centering
\caption{
    Width Reduction: The model size of the three different types of clients.
}
\begin{tabular}{llll} \toprule
Removed channels of Resnet20 (CIFAR-10) & Number of parameters ($p$) & Number of activations ($a$) & Capacity \\ \midrule
(Strong) - & 272,762 & 6,947,136 & 1.00 \\ 
(Moderate) $55\%$ of the consecutive channels & 52,874 & 3,039,552
 & 0.43 \\
(Weak) $80\%$ of the consecutive channels & 10,006 & 1,302,848 & 0.18 \\ \bottomrule
\end{tabular}
\label{tab:size2}
\end{table}

\subsection {Theoretical Analysis Preliminary}
For simplicity, we focus on the case where the model is partitioned to two subsets.
Our analysis can easily extend to the general case where the model is partitioned to more than 2 subsets.

\textbf{Notations} -- We first define a few notations for our analysis.
\begin{itemize}
    \item $m$: the total number of clients
    \item $s$: the number of \textit{strong} clients
    \item $T$: the number of total training communication rounds
    \item $\tau$: the model averaging interval (number of local steps)
    \item $\mathbf{x}$: the whole model parameters
    \item $\mathbf{y}$: the layers that are aggregated from all the clients
    \item $\mathbf{z}$: the layers that are aggregated from strong clients only
    \item $L_{y}$: the Lipschitz constant of $\mathbf{y}$
    \item $L_{z}$: the Lipschitz constant of $\mathbf{z}$
    \item $L_{max}$: the maximum Lipschitz constant across the two model partitions
\end{itemize}
$\mathbf{x}_{t,j}^{i} \in \mathbb{R}^d$ denotes the client $i$'s local model at $j^{th}$ local step in $t^{th}$ communication round.
The stochastic gradient computed from a single training data point $\xi_i$ is denoted by $\nabla f(\mathbf{x}_{t,j}^{i}, \xi_i)$.
For convenience, we use $\nabla f(\mathbf{x}_{t,j}^{i}$) instead.
The local full-batch gradient is denoted by $\nabla F_i(\cdot)$.
We use $\|\cdot\|$ to denote $\ell_2$ norm.

\textbf{Model Partitioning} -- 
We partition the model to two subsets and analyze their convergence properties separately.
First, the output-side layers that are trained by all the clients are denoted by $\mathbf{z}$.
All the remaining layers trained only by \textit{strong} clients are denoted by $\mathbf{y}$.
Thus, the total model parameters are defined as follows.
\begin{align}
    \mathbf{x} = (\mathbf{y}, \mathbf{z}). \nonumber
\end{align}
Likewise, the gradient of the model partitions are defined as follows.
\begin{align}
    \nabla f(\mathbf{x}) = (\nabla_y f(\mathbf{x}), \nabla_z f(\mathbf{x})), \nonumber
\end{align}
where $\nabla_y f(\mathbf{x})$ and $\nabla_z f(\mathbf{x})$ denote the stochastic gradients correspond to $\mathbf{y}$ and $\mathbf{z}$, respectively.
Finally, the local full-batch gradient and the global gradient for the model partitions are defined as follows.
\begin{align}
    \nabla F_i(\mathbf{x}) &= (\nabla_y F_i(\mathbf{x}), \nabla_z F_i(\mathbf{x})) \nonumber \\
        \nabla F(\mathbf{x}) &= (\nabla_y F(\mathbf{x}), \nabla_z F(\mathbf{x})) \nonumber
\end{align}

\textbf{Assumptions} --
Our analysis assumes the following.

1. (Smoothness). Each local objective function is $L_k$-smooth, $\forall k \in \{y, z\}$, that is, $\| \nabla_k F_i (\mathbf{x}) - \nabla_k F_i (\mathbf{x}') \| \leq L_k \| \mathbf{x} - \mathbf{x}' \|, \forall i \in \{ 1, \cdots, m \}$ and $\forall k \in \{y, z\}$.

2. (Unbiased Gradient). The stochastic gradient at each client is an unbiased estimator of the local full-batch gradient: $\mathop{\mathbb{E}}_{\xi_i} \left[ \nabla_k f(\mathbf{x}_{t,j}^{i}) \right] = \nabla_k F_i(\mathbf{x}_{t,j}^{i}), \forall i \in \{ 1, \cdots, m \}$ and $\forall k \in \{y, z\}$.

3. (Bounded Variance). The stochastic gradient at each client has bounded variance: $\mathop{\mathbb{E}}_{\xi_i} \left[ \| \nabla_k f(\mathbf{x}_{t,j}^{i}) - \nabla_k F_i (\mathbf{x}_{t,j}^{i}) \|^2 \right] \leq \sigma_k^2, \forall i \in \{ 1, \cdots, m \}$ and $\forall k \in \{y, z\}$.
Likewise, the full batch gradient has bounded variance: $\mathop{\mathbb{E}} \left[ \| \nabla_k F_i(\mathbf{x}_{t,j}^{i}) - \nabla_k F (\mathbf{x}_{t,j}^{i}) \|^2 \right] \leq \bar{\sigma}_k^2, \forall i \in \{ 1, \cdots, m \}$ and $\forall k \in \{y, z\}$.

\textbf{Averaging Weight} --
The proposed algorithm can be considered as having different averaging weights for the two model partitions as follows.
\begin{align}
    x_{t+1} = \mathbf{x}_t - \eta \sum_{i=1}^{m} \sum_{j=0}^{\tau - 1} p^i \nabla f(\mathbf{x}_{t,j}^{i}),
\end{align}
where $p^i$ is a client-specific averaging weight.
We define $p_y^i$ and $p_z^i$, the averaging weights for the two model partitions as follows.
\begin{align}
p_y^i &=
\begin{cases}
    \frac{1}{s}, \hspace{1cm} \textrm{ if client } i \textrm{ is a strong client} \\
    0, \hspace{1cm} \textrm{ if client } i \textrm{ is a weak client}
\end{cases} \label{piy}\\
p_z^i &= \frac{1}{m} \label{piz}
\end{align}

\subsection {Proof of Main Theorem and Lemmas} \label{app:proof}
To make this appendix self-contained, we define a few notations here again.
We consider federated optimization problems as follows.
\begin{align}
    \underset{\mathbf{x} \in \mathbb{R}^d}{\min}\left[F(\mathbf{x}) := \frac{1}{m} \sum_{i=1}^{m} F_i(\mathbf{x}) \right],
\end{align}
where $m$ is the number of local models and $F_i(\mathbf{x}) = \mathop{\mathbb{E}}_{\xi_i \sim D_i}[ F_i(\mathbf{x}, \xi_i)]$
is the local objective function associated with local data distribution $D_i$.
We define $\nabla_k F_i(\mathbf{x}, \xi_i)$ as the stochastic gradient with respect to the model partition $k$ and client $i$'s local data $\xi_i$. We omit $\xi_i$ for simplicity.
Our analysis result is shown in Theorem 1 as follows.

\textbf{Theorem 1.} \textit{
Suppose all $m$ local models are initialized to the same point $\mathbf{x}_0$. Under Assumption $1 \sim 3$, if Algorithm \ref{alg:EmbracingFL} runs for $T$ communication rounds and the learning rate satisfies $\eta \leq \min{ \left\{ \frac{1}{\tau L_{max}}, \frac{1}{4 L_{max} \sqrt{\tau (\tau - 1)}} \right\} }$, the average-squared gradient norm of $\mathbf{x}_t$ is bounded as follows.}
\begin{align}
    \frac{1}{T} \sum_{t=0}^{T-1} \mathop{\mathbb{E}}\left[ \left\| \nabla F(\mathbf{x}_t) \right\|^2 \right] & \leq \frac{14}{3 T \eta \tau} \left( F(\mathbf{x}_0) - F(\mathbf{x}_{*}) \right) \nonumber \\
    & \hspace{1cm} + \left( \frac{7 L_y \eta}{3s} + \frac{16 L_y^2 \eta^2 (\tau - 1)}{3} \right) \sigma_y^2  + \left( \frac{14}{3s} + \frac{64 L_y^2 \eta^2 \tau (\tau - 1)}{3} \right) \bar{\sigma}_y^2 \nonumber \\
    & \hspace{1cm} + \left( \frac{7L_z \eta}{3m} + \frac{8 L_z^2 \eta^2 (\tau - 1)}{3} \right) \sigma_z^2 \nonumber + \left( \frac{32 L_z^2 \eta^2 \tau (\tau - 1)}{3} \right) \bar{\sigma}_z^2 \nonumber
\end{align}
\begin{proof}
Based on Lemma \ref{lemma:framework} and \ref{lemma:discrepancy}, we have
\begin{align}
    \sum_{t=0}^{T-1} \mathop{\mathbb{E}}\left[ \left\| \nabla F(\mathbf{x}_t) \right\|^2 \right] & \leq \frac{2}{\eta \tau} \left( F(\mathbf{x}_0) - F(\mathbf{x}_{T-1}) \right) \nonumber \\
    & \hspace{1cm} + \frac{L_y\eta T}{s} \sigma_y^2 + \frac{2T}{s} \bar{\sigma}_y^2 + \frac{L_z \eta T}{m} \sigma_z^2 \nonumber \\
    &\hspace{1cm} + \sum_{t=0}^{T-1} \left( \frac{2 \eta^2 (\tau - 1) L_y^2}{1 - A_y} \sigma_y^2 + \frac{4A_y}{1 - A_y} \bar{\sigma}_y^2 + \frac{4A_y}{1 - A_y} \mathop{\mathbb{E}} \left[ \left\| \nabla_y F(\mathbf{x}_t) \right\|^2 \right] \right), \nonumber \\
    & \hspace{1cm} + \sum_{t=0}^{T-1} \left( \frac{\eta^2 (\tau - 1) L_z^2}{1 - A_z} \sigma_z^2 + \frac{2A_z}{1 - A_z} \bar{\sigma}_z^2 + \frac{2A_z}{1 - A_z} \mathop{\mathbb{E}} \left[ \left\| \nabla_z F(\mathbf{x}_t) \right\|^2 \right] \right), \nonumber
\end{align}
where $A_z \coloneqq 2\eta^2 L_z^2 \tau (\tau - 1) < 1$ and $A_y \coloneqq 2\eta^2 L_y^2 \tau (\tau - 1) < 1$.
Note that Lemma \ref{lemma:discrepancy} can be re-used to bound the model discrepancy of $\mathbf{z}$ by replacing all $\frac{1}{s}$ and $\sum_{i=1}^{s}$ with $\frac{1}{m}$ and $\sum_{i=1}^{m}$. 

After a minor rearrangement, we have
\begin{align}
    \sum_{t=0}^{T-1} \mathop{\mathbb{E}}\left[ \left\| \nabla F(\mathbf{x}_t) \right\|^2 \right] & \leq \frac{2}{\eta \tau} \left( F(\mathbf{x}_0) - F(\mathbf{x}_{T-1}) \right) \nonumber \\
    & \hspace{1cm} + \frac{L_y\eta T}{s} \sigma_y^2 + \frac{2T}{s} \bar{\sigma}_y^2 + \frac{L_z \eta T}{m} \sigma_z^2 \nonumber \\
    & \hspace{1cm} + \sum_{t=0}^{T-1} \left( \frac{2 \eta^2 (\tau - 1) L_y^2}{1 - A_y} \sigma_y^2 + \frac{4A_y}{1 - A_y} \bar{\sigma}_y^2 + \frac{\eta^2 (\tau - 1) L_z^2}{1 - A_z} \sigma_z^2 + \frac{2A_z}{1 - A_z} \bar{\sigma}_z^2 \right) \nonumber \\
    &\hspace{1cm} + \sum_{t=0}^{T-1} \left( \frac{4A_y}{1 - A_y} \mathop{\mathbb{E}} \left[ \left\| \nabla_y F(\mathbf{x}_t) \right\|^2 \right] + \frac{2A_z}{1 - A_z} \mathop{\mathbb{E}} \left[ \left\| \nabla_z F(\mathbf{x}_t) \right\|^2 \right] \right) \nonumber \\
    & \leq \frac{2}{\eta \tau} \left( F(\mathbf{x}_0) - F(\mathbf{x}_{T-1}) \right) \nonumber \\
    & \hspace{1cm} + \frac{L_y\eta T}{s} \sigma_y^2 + \frac{2T}{s} \bar{\sigma}_y^2 + \frac{L_z \eta T}{m} \sigma_z^2 \nonumber \\
    & \hspace{1cm} + \sum_{t=0}^{T-1} \left( \frac{2 \eta^2 (\tau - 1) L_y^2}{1 - A_y} \sigma_y^2 + \frac{4A_y}{1 - A_y} \bar{\sigma}_y^2 + \frac{\eta^2 (\tau - 1) L_z^2}{1 - A_z} \sigma_z^2 + \frac{2A_z}{1 - A_z} \bar{\sigma}_z^2 \right) \nonumber \\
    &\hspace{1cm} + \sum_{t=0}^{T-1} \left(\frac{4A_{max}}{1 - A_{max}} \mathop{\mathbb{E}} \left[ \left\| \nabla F(\mathbf{x}_t) \right\|^2 \right]\right), \label{amax}
\end{align}
where $A_{max} = \max\{A_y, A_z \}$.

If $A_{max} \leq \frac{1}{8}$, $\frac{4A_{max}}{1 - A_{max}} \leq \frac{4}{7}$.
Then, (\ref{amax}) can be simplified as follows.
\begin{align}
    \sum_{t=0}^{T-1} \mathop{\mathbb{E}}\left[ \left\| \nabla F(\mathbf{x}_t) \right\|^2 \right] & \leq \frac{2}{\eta \tau} \left( F(\mathbf{x}_0) - F(\mathbf{x}_{T-1}) \right) + \frac{L_y\eta T}{s} \sigma_y^2 + \frac{2T}{s} \bar{\sigma}_y^2 + \frac{L_z \eta T}{m} \sigma_z^2 \nonumber \\
    & \hspace{1cm} + \sum_{t=0}^{T-1} \left( \frac{2 \eta^2 (\tau - 1) L_y^2}{1 - A_y} \sigma_y^2 + \frac{4A_y}{1 - A_y} \bar{\sigma}_y^2 + \frac{\eta^2 (\tau - 1) L_z^2}{1 - A_z} \sigma_z^2 + \frac{2A_z}{1 - A_z} \bar{\sigma}_z^2 \right) \nonumber \\
    &\hspace{1cm} + \frac{4}{7} \sum_{t=0}^{T-1} \mathop{\mathbb{E}} \left[ \left\| \nabla F(\mathbf{x}_t) \right\|^2 \right] \nonumber
\end{align}
The same condition of $A_{max}$ also ensures $\frac{1}{1 - A_{max}} \leq \frac{8}{7}$.
Thus, we have
\begin{align}
    \sum_{t=0}^{T-1} \mathop{\mathbb{E}}\left[ \left\| \nabla F(\mathbf{x}_t) \right\|^2 \right] & \leq \frac{2}{\eta \tau} \left( F(\mathbf{x}_0) - F(\mathbf{x}_{T-1}) \right) \nonumber \\
    & \hspace{1cm} + \frac{L_y\eta T}{s} \sigma_y^2 + \frac{2T}{s} \bar{\sigma}_y^2 + \frac{L_z \eta T}{m} \sigma_z^2 \nonumber \\
    & \hspace{1cm} + \sum_{t=0}^{T-1} \left( \frac{16}{7} \eta^2 (\tau - 1) L_y^2 \sigma_y^2 + \frac{32}{7}A_y \bar{\sigma}_y^2 + \frac{8}{7}\eta^2 (\tau - 1) L_z^2 \sigma_z^2 + \frac{16}{7}A_z \bar{\sigma}_z^2 \right) \nonumber \\
    &\hspace{1cm} + \frac{4}{7} \sum_{t=0}^{T-1} \mathop{\mathbb{E}} \left[ \left\| \nabla F(\mathbf{x}_t) \right\|^2 \right]. \nonumber
\end{align}

After dividing both sides by $T$ and rearranging the terms, we finally have
\begin{align}
    \frac{1}{T} \sum_{t=0}^{T-1} \mathop{\mathbb{E}}\left[ \left\| \nabla F(\mathbf{x}_t) \right\|^2 \right] & \leq \frac{14}{3 T \eta \tau} \left( F(\mathbf{x}_0) - F(\mathbf{x}_{T-1}) \right) \nonumber \\
    & \hspace{1cm} + \frac{7 L_y\eta}{3s} \sigma_y^2  + \frac{14}{3s} \bar{\sigma}_y^2 + \frac{7L_z \eta}{3m} \sigma_z^2 \nonumber \\
    & \hspace{1cm} + \frac{16}{3} \eta^2 (\tau - 1) L_y^2 \sigma_y^2 + \frac{32}{3}A_y \bar{\sigma}_y^2 + \frac{8}{3}\eta^2 (\tau - 1) L_z^2 \sigma_z^2 + \frac{16}{3}A_z \bar{\sigma}_z^2 \nonumber \\
    & = \frac{14}{3 T \eta \tau} \left( F(\mathbf{x}_0) - F(\mathbf{x}_{T-1}) \right) \nonumber \\
    & \hspace{1cm} + \left( \frac{7 L_y \eta}{3s} + \frac{16 L_y^2 \eta^2 (\tau - 1)}{3} \right) \sigma_y^2  + \left( \frac{14}{3s} + \frac{64 L_y^2 \eta^2 \tau (\tau - 1)}{3} \right) \bar{\sigma}_y^2 \nonumber \\
    & \hspace{1cm} + \left( \frac{7L_z \eta}{3m} + \frac{8 L_z^2 \eta^2 (\tau - 1)}{3} \right) \sigma_z^2 \nonumber + \left( \frac{32 L_z^2 \eta^2 \tau (\tau - 1)}{3} \right) \bar{\sigma}_z^2 \nonumber \\
    & \leq \frac{14}{3 T \eta \tau} \left( F(\mathbf{x}_0) - F(\mathbf{x}_{*}) \right) \nonumber \\
    & \hspace{1cm} + \left( \frac{7 L_y \eta}{3s} + \frac{16 L_y^2 \eta^2 (\tau - 1)}{3} \right) \sigma_y^2  + \left( \frac{14}{3s} + \frac{64 L_y^2 \eta^2 \tau (\tau - 1)}{3} \right) \bar{\sigma}_y^2 \nonumber \\
    & \hspace{1cm} + \left( \frac{7L_z \eta}{3m} + \frac{8 L_z^2 \eta^2 (\tau - 1)}{3} \right) \sigma_z^2 \nonumber + \left( \frac{32 L_z^2 \eta^2 \tau (\tau - 1)}{3} \right) \bar{\sigma}_z^2 \nonumber
\end{align}
where $\mathbf{x}_{*}$ is a local minimum.
This completes the proof.
\end{proof}

\textbf{Learning Rate Constraints} -- Theorem \ref{theorem:main} has two learning rate constraints as follows.
\begin{align}
    L_{max}\eta \tau - 1 &\leq 0 \nonumber \\
    2 L_{max}^2 \eta^2 \tau (\tau - 1) &\leq \frac{1}{8} \nonumber
\end{align}
After a minor rearrangement, we can have a single learning rate constraint as follows.
\begin{align}
    \eta \leq \min{ \left\{ \frac{1}{\tau L_{max}}, \frac{1}{4 L_{max} \sqrt{\tau (\tau - 1)}} \right\} }
\end{align}

\begin{lemma}
\label{lemma:framework}
(framework) Under Assumption $1 \sim 3$, if the learning rate satisfies $\eta \leq 1 / (\tau L_{max})$, Algorithm 1 ensures
\begin{align}
    \sum_{t=0}^{T-1} \mathop{\mathbb{E}}\left[ \left\| \nabla F(\mathbf{x}_t) \right\|^2 \right] & \leq \frac{2}{\eta \tau} \left( F(\mathbf{x}_0) - F(\mathbf{x}_{T-1}) \right) \nonumber \\
    & \hspace{1cm} + \frac{L_y\eta T}{s} \sigma_y^2 + \frac{2T}{s} \bar{\sigma}_y^2 + \frac{L_z \eta T}{m} \sigma_z^2 \nonumber \\
    &\hspace{1cm} + \frac{2 L_y^2}{s \tau} \sum_{t=0}^{T-1} \sum_{i=1}^{s} \sum_{j=0}^{\tau-1} \mathop{\mathbb{E}} \left[ \left\| \mathbf{y}_{t,j}^{i} - \mathbf{y}_t \right\|^2 \right] \nonumber \\
    & \hspace{1cm} + \frac{L_z^2}{m \tau} \sum_{t=0}^{T-1} \sum_{i=1}^{m} \sum_{j=0}^{\tau-1} \mathop{\mathbb{E}} \left[ \left\| \mathbf{z}_{t,j}^{i} - \mathbf{z}_t \right\|^2 \right] \nonumber
\end{align}
\end{lemma}
\begin{proof}
Based on Assumption 1, we have
\begin{align}
    \mathop{\mathbb{E}}\left[ F(\mathbf{x}_{t+1}) - F(\mathbf{x}_{t}) \right] & \leq \mathop{\mathbb{E}}\left[ \langle \nabla F(\mathbf{x}_{t}), \mathbf{x}_{t+1} - \mathbf{x}_{t} \rangle \right] + \frac{L}{2} \left\| \mathbf{x}_{t+1} - \mathbf{x}_{t} \right\|^2 \nonumber \\
    & = \mathop{\mathbb{E}}\left[ \langle \nabla_y F(\mathbf{x}_{t}), \mathbf{y}_{t+1} - \mathbf{y}_{t} \rangle \right] + \frac{L_y}{2} \left\| \mathbf{y}_{t+1} - \mathbf{y}_{t} \right\|^2 \label{decompose} \\
    & \quad\quad + \mathop{\mathbb{E}}\left[ \langle \nabla_z F(\mathbf{x}_{t}), \mathbf{z}_{t+1} - \mathbf{z}_{t} \rangle \right] + \frac{L_z}{2} \left\| \mathbf{z}_{t+1} - \mathbf{z}_{t} \right\|^2 \nonumber
\end{align}
For convenience, we define the gradients accumulated within a communication round as follows.
\begin{align}
    \Delta_t^i & \coloneqq \sum_{j=0}^{\tau-1} \nabla f(\mathbf{x}_{t,j}^i) \nonumber \\
    \Delta_t & \coloneqq \sum_{i=1}^{m} p^i \Delta_t^i \nonumber
\end{align}
Based on (\ref{piy}) and (\ref{piz}), the accumulated gradients can be re-written as follows.
\begin{align}
    \Delta_{t}^{i} &= \left( \Delta_{y,t}^i, \Delta_{z,t}^i \right) \coloneqq \left( \sum_{j=0}^{\tau-1} \nabla_y f(\mathbf{x}_{t,j}^i), \sum_{j=0}^{\tau-1} \nabla_z f(\mathbf{x}_{t,j}^i) \right) \nonumber \\
    \Delta_{t} &= \left( \Delta_{y,t}, \Delta_{z,t} \right) \coloneqq \left( \frac{1}{s} \sum_{i=1}^{s} \Delta_{y,t}^i, \frac{1}{m} \sum_{i=1}^{m} \Delta_{z,t}^i \right) \nonumber
\end{align}
Using the above definitions, (\ref{decompose}) can be simplified as follows.
\begin{align}
    \mathop{\mathbb{E}}\left[ F(\mathbf{x}_{t+1}) - F(\mathbf{x}_{t}) \right] & \leq -\eta \mathop{\mathbb{E}}\left[ \langle \nabla_y F(\mathbf{x}_t), \Delta_{y,t} \rangle \right] + \frac{L_y \eta^2}{2} \mathop{\mathbb{E}}\left[ \left\| \Delta_{y,t} \right\|^2 \right] \nonumber \\
    & \quad\quad -\eta \mathop{\mathbb{E}}\left[ \langle \nabla_z F(\mathbf{x}_t), \Delta_{z,t} \rangle \right] + \frac{L_z \eta^2}{2} \mathop{\mathbb{E}}\left[ \left\| \Delta_{z,t} \right\|^2 \right] \nonumber \\
    & = -\eta \mathop{\mathbb{E}}\left[ \langle \nabla_y F(\mathbf{x}_t), \Delta_{y,t} + \tau \nabla_y F(\mathbf{x}_t) - \tau \nabla_y F(\mathbf{x}_t) \rangle \right] \nonumber \\
    & \quad\quad -\eta \mathop{\mathbb{E}}\left[ \langle \nabla_z F(\mathbf{x}_t), \Delta_{z,t} + \tau \nabla_z F(\mathbf{x}_t) - \tau \nabla_z F(\mathbf{x}_t) \rangle \right]\nonumber \\
    & \quad\quad + \frac{L_y \eta^2}{2} \mathop{\mathbb{E}}\left[ \left\| \Delta_{y,t} \right\|^2 \right] + \frac{L_z \eta^2}{2} \mathop{\mathbb{E}}\left[ \left\| \Delta_{z,t} \right\|^2 \right] \nonumber \\
    & = - \eta \tau \mathop{\mathbb{E}}\left[ \left\| \nabla_y F(\mathbf{x}_{t}) \right\|^2 \right] + \eta \underset{T_1}{\underbrace{ \mathop{\mathbb{E}} \left[ \langle - \nabla_y F(\mathbf{x}_t), \Delta_{y,t} - \tau \nabla_y F(\mathbf{x}_t) \rangle \right] }} \nonumber \\
    & \quad\quad - \eta \tau \mathop{\mathbb{E}}\left[ \left\| \nabla_z F(\mathbf{x}_{t}) \right\|^2 \right] + \eta \underset{T_2}{\underbrace{ \mathop{\mathbb{E}} \left[ \langle - \nabla_z F(\mathbf{x}_t), \Delta_{z,t} - \tau \nabla_z F(\mathbf{x}_t) \rangle \right] }} \nonumber \\
    & \quad\quad + \frac{L_y \eta^2}{2} \underset{T_3}{ \underbrace{ \mathop{\mathbb{E}}\left[ \left\| \Delta_{y,t} \right\|^2 \right]}} + \frac{L_z \eta^2}{2} \underset{T_4}{\underbrace{ \mathop{\mathbb{E}}\left[ \left\| \Delta_{z,t} \right\|^2 \right]}} \label{framework}
\end{align}
Now, we bound $T_1 \sim T_4$, separately.

\textbf{Bounding $T_1$} 
\begin{align}
    T_1 & = \mathop{\mathbb{E}} \left[ \langle - \nabla_y F(\mathbf{x}_t), \Delta_{y,t} - \tau \nabla_y F(\mathbf{x}_t) \rangle \right] \nonumber \\
    & = \langle - \nabla_y F(\mathbf{x_t}), \mathop{\mathbb{E}} \left[ \sum_{i=1}^{m} \sum_{j=0}^{\tau - 1} p^i \nabla_y f(\mathbf{x}_{t,j}^i) - \tau \nabla_y F(\mathbf{x}_t) \right] \rangle \nonumber \\
    & = \langle - \nabla_y F(\mathbf{x_t}), \mathop{\mathbb{E}} \left[ \sum_{i=1}^{m} \sum_{j=0}^{\tau - 1} p^i \nabla_y F_i(\mathbf{x}_{t,j}^i) - \tau \nabla_z F(\mathbf{x}_t) \right] \rangle \nonumber \\
    & = \langle - \sqrt{\tau} \nabla_y F(\mathbf{x_t}), \mathop{\mathbb{E}} \left[ \frac{1}{\sqrt{\tau}} \sum_{i=1}^{m} \sum_{j=0}^{\tau - 1} p^i \nabla_y F_i(\mathbf{x}_{t,j}^i) - \sqrt{\tau} \nabla_y F(\mathbf{x}_t) \right] \rangle \nonumber \\
    & = \frac{\tau}{2} \left\| \nabla_y F(\mathbf{x}_t) \right\|^2 - \frac{1}{2} \mathop{\mathbb{E}} \left[ \left\| \frac{1}{\sqrt{\tau}} \sum_{i=1}^{m} \sum_{j=0}^{\tau - 1} p^i \nabla_y F_i(\mathbf{x}_{t,j}^i) \right\|^2 \right] \nonumber \\
    & \hspace{1cm} + \frac{1}{2} \mathop{\mathbb{E}} \left[ \left\| \frac{1}{\sqrt{\tau}} \sum_{i=1}^{m} \sum_{j=0}^{\tau - 1} p^i \nabla_y F_i(\mathbf{x}_{t,j}^i) - \sqrt{\tau} \nabla_y F(\mathbf{x}_t) \right\|^2 \right], \label{t1_long} \\
    & = \frac{\tau}{2} \left\| \nabla_y F(\mathbf{x}_t) \right\|^2 - \frac{1}{2} \mathop{\mathbb{E}} \left[ \left\| \frac{1}{\sqrt{\tau}} \sum_{i=1}^{m} \sum_{j=0}^{\tau - 1} p^i \nabla_y F_i(\mathbf{x}_{t,j}^i) \right\|^2 \right] \nonumber \\
    & \hspace{1cm} + \frac{1}{2} \mathop{\mathbb{E}} \left[ \left\| \frac{1}{\sqrt{\tau}} \sum_{i=1}^{m} \sum_{j=0}^{\tau - 1} p^i \nabla_y F_i(\mathbf{x}_{t,j}^i) - \sqrt{\tau} \left( \nabla_y F(\mathbf{x}_t) - \nabla_y F_s(\mathbf{x}_t) + \nabla_y F_s(\mathbf{x}_t) \right) \right\|^2 \right], \nonumber \\
    & \leq \frac{\tau}{2} \left\| \nabla_y F(\mathbf{x}_t) \right\|^2 - \frac{1}{2} \mathop{\mathbb{E}} \left[ \left\| \frac{1}{\sqrt{\tau}} \sum_{i=1}^{m} \sum_{j=0}^{\tau - 1} p^i \nabla_y F_i(\mathbf{x}_{t,j}^i) \right\|^2 \right] \nonumber \\
    & \hspace{1cm} + \mathop{\mathbb{E}} \left[ \left\| \frac{1}{\sqrt{\tau}} \sum_{i=1}^{m} \sum_{j=0}^{\tau - 1} p^i \nabla_y F_i(\mathbf{x}_{t,j}^i) - \sqrt{\tau} \nabla_y F_s(\mathbf{x}_t) \right\|^2 \right] \nonumber \\
    & \hspace{1cm} + \mathop{\mathbb{E}}\left[ \left\| \sqrt{\tau} \nabla_y F_s(\mathbf{x}_t) - \sqrt{\tau} \nabla_y F(\mathbf{x}_t) \right\|^2 \right] \label{t1_jensen} \\
    & \leq \frac{\tau}{2} \left\| \nabla_y F(\mathbf{x}_t) \right\|^2 - \frac{1}{2} \mathop{\mathbb{E}} \left[ \left\| \frac{1}{\sqrt{\tau}} \sum_{i=1}^{m} \sum_{j=0}^{\tau - 1} p^i \nabla_y F_i(\mathbf{x}_{t,j}^i) \right\|^2 \right] \nonumber \\
    & \hspace{1cm} + \mathop{\mathbb{E}} \left[ \left\| \frac{1}{\sqrt{\tau}} \sum_{i=1}^{m} \sum_{j=0}^{\tau - 1} p^i \nabla_y F_i(\mathbf{x}_{t,j}^i) - \sqrt{\tau} \nabla_y F_s(\mathbf{x}_t) \right\|^2 \right] + \frac{\tau}{s} \bar{\sigma}_y^2 \label{t1_sigma},
\end{align}
where (\ref{t1_long}) holds based on $\langle \mathbf{a}, \mathbf{b} \rangle = \frac{1}{2}\{ \left\| \mathbf{a} \right\|^2 + \left\| \mathbf{b} \right\|^2 - \left\| \mathbf{a} - \mathbf{b} \right\|^2 \}$; (\ref{t1_jensen}) is based on the convexity of $\ell_2$ norm and Jensen's inequality; (\ref{t1_sigma}) is based on Lemma \ref{subgrad}.

Because $p^i$ is $\frac{1}{s}$ for the strong clients and $0$ for the other clients, (\ref{t1_sigma}) can be re-written as follows.
\begin{align}
    T_1 & \leq \frac{\tau}{2} \left\| \nabla_y F(\mathbf{x}_t) \right\|^2 - \frac{1}{2 s^2 \tau} \mathop{\mathbb{E}} \left[ \left\| \sum_{i=1}^{s} \sum_{j=0}^{\tau - 1} \nabla_y F_i(\mathbf{x}_{t,j}^i) \right\|^2 \right] \nonumber \\
    & \hspace{1cm} + \mathop{\mathbb{E}} \left[ \left\| \frac{1}{\sqrt{\tau}} \sum_{i=1}^{s} \sum_{j=0}^{\tau - 1} \frac{1}{s} \nabla_y F_i(\mathbf{x}_{t,j}^i) - \sqrt{\tau} \nabla_y F_s(\mathbf{x}_t) \right\|^2 \right] + \frac{\tau}{s} \bar{\sigma}_y^2 \nonumber \\
    & \leq \frac{\tau}{2} \left\| \nabla_y F(\mathbf{x}_t) \right\|^2 - \frac{1}{2 s^2 \tau} \mathop{\mathbb{E}} \left[ \left\| \sum_{i=1}^{s} \sum_{j=0}^{\tau - 1} \nabla_y F_i(\mathbf{x}_{t,j}^i) \right\|^2 \right] \nonumber \\
    & \hspace{1cm} + \mathop{\mathbb{E}} \left[ \left\| \frac{1}{s \sqrt{\tau}} \sum_{i=1}^{s} \sum_{j=0}^{\tau - 1} \left( \nabla_y F_i(\mathbf{x}_{t,j}^i) - \nabla_y F_i(\mathbf{x}_t) \right) \right\|^2 \right] + \frac{\tau}{s} \bar{\sigma}_y^2 \nonumber \\
    & \leq \frac{\tau}{2} \left\| \nabla_y F(\mathbf{x}_t) \right\|^2 - \frac{1}{2 s^2 \tau} \mathop{\mathbb{E}} \left[ \left\| \sum_{i=1}^{s} \sum_{j=0}^{\tau - 1} \nabla_y F_i(\mathbf{x}_{t,j}^i) \right\|^2 \right] \nonumber \\
    & \hspace{1cm} + \frac{1}{s} \sum_{i=1}^{s} \sum_{j=0}^{\tau - 1} \mathop{\mathbb{E}} \left[ \left\| \nabla_y F_i(\mathbf{x}_{t,j}^i) - \nabla_y F_i(\mathbf{x}_t) \right\|^2 \right] + \frac{\tau}{s} \bar{\sigma}_y^2 \label{t1_jensen2} \\
    & \leq \frac{\tau}{2} \left\| \nabla_y F(\mathbf{x}_t) \right\|^2 - \frac{1}{2 s^2 \tau} \mathop{\mathbb{E}} \left[ \left\| \sum_{i=1}^{s} \sum_{j=0}^{\tau - 1} \nabla_y F_i(\mathbf{x}_{t,j}^i) \right\|^2 \right] + \frac{L_y^2}{s} \sum_{i=1}^{s} \sum_{j=0}^{\tau - 1} \mathop{\mathbb{E}} \left[ \left\| \mathbf{y}_{t,j}^i - \mathbf{y}_t \right\|^2 \right] + \frac{\tau}{s} \bar{\sigma}_y^2, \label{t1_final}
\end{align}
where (\ref{t1_jensen2}) is based on the convexity of $\ell_2$ norm and Jensen's inequality; (\ref{t1_final}) holds based on Assumption 1.

\textbf{Bounding $T_2$}
\begin{align}
    T_2 & = \mathop{\mathbb{E}} \left[ \langle - \nabla_z F(\mathbf{x}_t), \Delta_{z,t} - \tau \nabla_z F(\mathbf{x}_t) \rangle \right] \nonumber \\
    & = \langle - \nabla_z F(\mathbf{x}_t), \mathop{\mathbb{E}} \left[ \frac{1}{m} \sum_{i=1}^{m} \sum_{j=0}^{\tau - 1} \nabla_z f(\mathbf{x}_{t,j}^{i}) - \tau \nabla_z F(\mathbf{x}_{t}) \right] \rangle \nonumber \\
    & = \langle - \nabla_z F(\mathbf{x}_t), \mathop{\mathbb{E}} \left[ \frac{1}{m} \sum_{i=1}^{m} \sum_{j=0}^{\tau - 1} \nabla_z F_i (\mathbf{x}_{t,j}^{i}) - \tau \nabla_z F(\mathbf{x}_{t}) \right] \rangle \nonumber \\
    & = \langle - \sqrt{\tau} \nabla_z F(\mathbf{x}_t), \mathop{\mathbb{E}} \left[ \frac{1}{m \sqrt{\tau}} \sum_{i=1}^{m} \sum_{j=0}^{\tau - 1} \nabla_z F_i (\mathbf{x}_{t,j}^{i}) - \sqrt{\tau} \nabla_z F(\mathbf{x}_{t}) \right] \rangle \nonumber \\
    & = \frac{\tau}{2} \left\| \nabla_z F(\mathbf{x}_t) \right\|^2 + \frac{1}{2} \mathop{\mathbb{E}} \left[ \left\| \frac{1}{m \sqrt{\tau}} \sum_{i=1}^{m} \sum_{j=0}^{\tau - 1} \nabla_z F_i(\mathbf{x}_{t,j}^{i}) - \sqrt{\tau} \nabla_z F(\mathbf{x}_t) \right\|^2 \right]  - \frac{1}{2 m^2 \tau} \mathop{\mathbb{E}} \left[ \left\| \sum_{i=1}^{m} \sum_{j=0}^{\tau - 1} \nabla_z F_i (\mathbf{x}_{t,j}^{i}) \right\|^2 \right] \label{t2_long} \\
    & = \frac{\tau}{2} \left\| \nabla_z F(\mathbf{x}_t) \right\|^2 + \frac{1}{2} \mathop{\mathbb{E}} \left[ \left\| \frac{1}{m\sqrt{\tau}} \left( \sum_{i=1}^{m} \sum_{j=0}^{\tau - 1} \left( \nabla_z F_i(\mathbf{x}_{t,j}^{i}) - \nabla_z F_i(\mathbf{x}_t) \right) \right) \right\|^2 \right]  - \frac{1}{2 m^2 \tau}  \mathop{\mathbb{E}} \left[ \left\| \sum_{i=1}^{m} \sum_{j=0}^{\tau - 1} \nabla_z F_i (\mathbf{x}_{t,j}^{i}) \right\|^2 \right] \nonumber\\
    & = \frac{\tau}{2} \left\| \nabla_z F(\mathbf{x}_t) \right\|^2 + \frac{1}{2 m^2 \tau} \mathop{\mathbb{E}} \left[ \left\| \sum_{i=1}^{m} \sum_{j=0}^{\tau - 1} \left( \nabla_z F_i(\mathbf{x}_{t,j}^{i}) - \nabla_z F_i(\mathbf{x}_t) \right) \right\|^2 \right]  - \frac{1}{2 m^2 \tau} \mathop{\mathbb{E}} \left[ \left\| \sum_{i=1}^{m} \sum_{j=0}^{\tau - 1} \nabla_z F_i (\mathbf{x}_{t,j}^{i}) \right\|^2 \right] \nonumber\\
    & \leq \frac{\tau}{2} \left\| \nabla_z F(\mathbf{x}_t) \right\|^2 + \frac{1}{2m} \sum_{i=1}^{m} \sum_{j=0}^{\tau - 1} \mathop{\mathbb{E}} \left[ \left\| \nabla_z F_i(\mathbf{x}_{t,j}^{i}) - \nabla_z F_i(\mathbf{x}_t) \right\|^2 \right]  - \frac{1}{2 m^2 \tau} \mathop{\mathbb{E}} \left[ \left\| \sum_{i=1}^{m} \sum_{j=0}^{\tau - 1} \nabla_z F_i (\mathbf{x}_{t,j}^{i}) \right\|^2 \right] \label{t2_jensen} \\
    & \leq \frac{\tau}{2} \left\| \nabla_z F(\mathbf{x}_t) \right\|^2 + \frac{L_{z}^2}{2m} \sum_{i=1}^{m} \sum_{j=0}^{\tau - 1} \mathop{\mathbb{E}} \left[ \left\| \mathbf{z}_{t,j}^{i} - \mathbf{z}_t \right\|^2 \right]  - \frac{1}{2 m^2 \tau} \mathop{\mathbb{E}} \left[ \left\| \sum_{i=1}^{m} \sum_{j=0}^{\tau - 1} \nabla_z F_i (\mathbf{x}_{t,j}^{i}) \right\|^2 \right], \label{t2_final}
\end{align}
where (\ref{t2_long}) holds based on $\langle \mathbf{a}, \mathbf{b} \rangle = \frac{1}{2}\{ \left\| \mathbf{a} \right\|^2 + \left\| \mathbf{b} \right\|^2 - \left\| \mathbf{a} - \mathbf{b} \right\|^2 \}$; (\ref{t2_jensen}) follows from the convexity of $\ell_2$ norm and Jensen's inequality; (\ref{t2_final}) is based on Assumption 1.

\textbf{Bounding $T_3$}
\begin{align}
    T_3 &= \mathop{\mathbb{E}} \left[ \left\| \Delta_{y,t} \right\|^2 \right] \nonumber\\
    & = \mathop{\mathbb{E}}\left[ \left\| \sum_{i=1}^{m} p^i \Delta_{y,t}^i \right\|^2 \right] \nonumber\\
    & = \mathop{\mathbb{E}}\left[ \left\| \frac{1}{s} \sum_{i=1}^{s} \Delta_{y,t}^i \right\|^2 \right] \nonumber\\
    & = \frac{1}{s^2} \mathop{\mathbb{E}} \left[ \left\| \sum_{i=1}^{s} \sum_{j=0}^{\tau - 1} \nabla_y f(\mathbf{x}_{t,j}^{i}) \right\|^2 \right] \nonumber \\
    & = \frac{1}{s^2} \mathop{\mathbb{E}} \left[ \left\| \sum_{i=1}^{s} \sum_{j=0}^{\tau - 1} \left( \nabla_y f(\mathbf{x}_{t,j}^{i}) - \nabla_y F_i(\mathbf{x}_{t,j}^{i}) \right) \right\|^2 \right] + \frac{1}{s^2} \mathop{\mathbb{E}} \left[ \left\| \sum_{i=1}^{s} \sum_{j=0}^{\tau - 1} \nabla_y F_i (\mathbf{x}_{t,j}^{i}) \right\|^2 \right] \label{t3_mean} \\
    & \leq \frac{\tau}{s} \sigma_y^2 + \frac{1}{s^2} \mathop{\mathbb{E}} \left[ \left\| \sum_{i=1}^{s} \sum_{j=0}^{\tau - 1} \nabla_y F_i (\mathbf{x}_{t,j}^{i}) \right\|^2 \right], \label{t3_final}
\end{align}
where (\ref{t3_mean}) is based on a simple equation: $\mathop{\mathbb{E}}\left[ \left\| \mathbf{x} \right\|^2 \right] = \mathop{\mathbb{E}}\left[ \left\| \mathbf{x} - \mathop{\mathbb{E}}\left[ \mathbf{x}\right] \right\|^2 \right] + \left\| \mathop{\mathbb{E}}\left[ \mathbf{x} \right] \right\|^2$ and (\ref{t3_final}) is based on Assumption 3 and because $\nabla_y f(\mathbf{x}_{t,j}^{i}) - \nabla_y F(\mathbf{x}_{t,j}^{i})$ has a mean of 0 and independent across $s$ clients.

\textbf{Bounding $T_4$}
\begin{align}
    T_4 &= \mathop{\mathbb{E}} \left[ \left\| \Delta_{z,t} \right\|^2 \right] \nonumber\\
    & = \mathop{\mathbb{E}}\left[ \left\| \frac{1}{m} \sum_{i=1}^{m} \Delta_{z,t}^i \right\|^2 \right] \nonumber\\
    & = \frac{1}{m^2} \mathop{\mathbb{E}} \left[ \left\| \sum_{i=1}^{m} \sum_{j=0}^{\tau - 1} \nabla_z f(\mathbf{x}_{t,j}^{i}) \right\|^2 \right] \nonumber \\
    & = \frac{1}{m^2} \mathop{\mathbb{E}} \left[ \left\| \sum_{i=1}^{m} \sum_{j=0}^{\tau - 1} \left( \nabla_z f(\mathbf{x}_{t,j}^{i}) - \nabla_z F_i(\mathbf{x}_{t,j}^{i}) \right) \right\|^2 \right] + \frac{1}{m^2} \mathop{\mathbb{E}} \left[ \left\| \sum_{i=1}^{m} \sum_{j=0}^{\tau - 1} \nabla_z F_i (\mathbf{x}_{t,j}^{i}) \right\|^2 \right] \label{t4_mean} \\
    & \leq \frac{\tau}{m} \sigma_z^2 + \frac{1}{m^2} \mathop{\mathbb{E}} \left[ \left\| \sum_{i=1}^{m} \sum_{j=0}^{\tau - 1} \nabla_z F_i (\mathbf{x}_{t,j}^{i}) \right\|^2 \right], \label{t4_final}
\end{align}
where (\ref{t4_mean}) is based on a simple equation: $\mathop{\mathbb{E}}\left[ \left\| \mathbf{x} \right\|^2 \right] = \mathop{\mathbb{E}}\left[ \left\| \mathbf{x} - \mathop{\mathbb{E}}\left[ \mathbf{x}\right] \right\|^2 \right] + \left\| \mathop{\mathbb{E}}\left[ \mathbf{x} \right] \right\|^2$ and (\ref{t4_final}) holds based on Assumption 3 and because $\nabla_z f(\mathbf{x}_{t,j}^{i}) - \nabla_z F_i(\mathbf{x}_{t,j}^{i})$ has a mean of 0 and independent across $m$ clients.

Plugging in (\ref{t1_final}), (\ref{t2_final}), (\ref{t3_final}), and (\ref{t4_final}) into (\ref{framework}), we have
\begin{align}
    \mathop{\mathbb{E}}\left[ F(\mathbf{x}_{t+1}) - F(\mathbf{x}_{t}) \right] & \leq - \frac{\eta \tau}{2} \mathop{\mathbb{E}}\left[ \left\| \nabla_y F(\mathbf{x}_t) \right\|^2 \right] - \frac{\eta \tau}{2} \mathop{\mathbb{E}}\left[ \left\| \nabla_z F(\mathbf{x}_t) \right\|^2 \right] \nonumber \\
    &\hspace{1cm} + \frac{L_y^2 \eta}{s} \sum_{i=1}^{s} \sum_{j=0}^{\tau-1} \mathop{\mathbb{E}} \left[ \left\| \mathbf{y}_{t,j}^{i} - \mathbf{y}_t \right\|^2 \right] + \frac{L_z^2 \eta}{2m} \sum_{i=1}^{m} \sum_{j=0}^{\tau-1} \mathop{\mathbb{E}} \left[ \left\| \mathbf{z}_{t,j}^{i} - \mathbf{z}_t \right\|^2 \right] \nonumber \\
    &\hspace{1cm} + \frac{L_y \eta^2 \tau}{2s} \sigma_y^2 + \frac{\eta \tau}{s} \bar{\sigma}_y^2 + \frac{L_y \eta^2 \tau - \eta}{2 s^2 \tau} \mathop{\mathbb{E}} \left[ \left\| \sum_{i=1}^{s} \sum_{j=0}^{\tau-1} \nabla_y F_i (\mathbf{x}_{t,j}^{i}) \right\|^2 \right] \nonumber \\
    &\hspace{1cm} + \frac{L_z \eta^2 \tau}{2m} \sigma_z^2 + \frac{L_z \eta^2 \tau - \eta}{2 m^2 \tau} \mathop{\mathbb{E}} \left[ \left\| \sum_{i=1}^{m} \sum_{j=0}^{\tau-1} \nabla_z F_i (\mathbf{x}_{t,j}^{i}) \right\|^2 \right] \label{plugin}
\end{align}
If $\eta \leq 1 / (\tau L_{max})$, (\ref{plugin}) can be simplified as follows.
\begin{align}
    \mathop{\mathbb{E}}\left[ F(\mathbf{x}_{t+1}) - F(\mathbf{x}_{t}) \right] & \leq - \frac{\eta \tau}{2} \mathop{\mathbb{E}}\left[ \left\| \nabla_y F(\mathbf{x}_t) \right\|^2 \right] - \frac{\eta \tau}{2} \mathop{\mathbb{E}}\left[ \left\| \nabla_z F(\mathbf{x}_t) \right\|^2 \right] \nonumber \\
    &\hspace{1cm} + \frac{L_y \eta^2 \tau}{2s} \sigma_y^2 + \frac{\eta \tau}{s} \bar{\sigma}_y^2 + \frac{L_z \eta^2 \tau}{2m} \sigma_z^2 \nonumber \\
    &\hspace{1cm} + \frac{L_y^2 \eta}{s} \sum_{i=1}^{s} \sum_{j=0}^{\tau-1} \mathop{\mathbb{E}} \left[ \left\| \mathbf{y}_{t,j}^{i} - \mathbf{y}_t \right\|^2 \right] + \frac{L_z^2 \eta}{2m} \sum_{i=1}^{m} \sum_{j=0}^{\tau-1} \mathop{\mathbb{E}} \left[ \left\| \mathbf{z}_{t,j}^{i} - \mathbf{z}_t \right\|^2 \right] \label{plugin2}
\end{align}
Summing up (\ref{plugin2}) across all the communication rounds: $t \in \{0, \cdots, T-1 \}$, we have a telescoping sum as follows.
\begin{align}
    \sum_{t=0}^{T-1} \left( F(\mathbf{x}_{t+1}) - F(\mathbf{x}_t) \right) &\leq -\frac{\eta \tau}{2} \sum_{t=0}^{T-1} \mathop{\mathbb{E}}\left[ \left\| \nabla_y F(\mathbf{x}_t) \right\|^2 \right] - \frac{\eta \tau}{2} \sum_{t=0}^{T-1} \mathop{\mathbb{E}}\left[ \left\| \nabla_z F(\mathbf{x}_t) \right\|^2 \right] \nonumber \\
    &\hspace{1cm} + \frac{L_y T \eta^2 \tau}{2s} \sigma_y^2 + \frac{T \eta \tau}{s} \bar{\sigma}_y^2 + \frac{L_z \eta^2 T \tau}{2m} \sigma_z^2 \nonumber \\
    &\hspace{1cm} + \frac{L_y^2 \eta}{s} \sum_{t=0}^{T-1} \sum_{i=1}^{s} \sum_{j=0}^{\tau-1} \mathop{\mathbb{E}} \left[ \left\| \mathbf{y}_{t,j}^{i} - \mathbf{y}_t \right\|^2 \right] \nonumber \\
    &\hspace{1cm} + \frac{L_z^2 \eta}{2m} \sum_{t=0}^{T-1} \sum_{i=1}^{m} \sum_{j=0}^{\tau-1} \mathop{\mathbb{E}} \left[ \left\| \mathbf{z}_{t,j}^{i} - \mathbf{z}_t \right\|^2 \right] \nonumber
\end{align}
Because $\| \nabla_y F(\mathbf{x}) \|^2 + \| \nabla_z F(\mathbf{x}) \|^2 = \| \nabla F(\mathbf{x}) \|^2$, we can simplify the above bound as follows.
\begin{align}
    \sum_{t=0}^{T-1} \left( F(\mathbf{x}_{t+1}) - F(\mathbf{x}_t) \right) &\leq -\frac{\eta \tau}{2} \sum_{t=0}^{T-1} \mathop{\mathbb{E}}\left[ \left\| \nabla F(\mathbf{x}_t) \right\|^2 \right] \nonumber \\
    &\hspace{1cm} + \frac{L_y T \eta^2 \tau}{2s} \sigma_y^2 + \frac{T \eta \tau}{s} \bar{\sigma}_y^2 + \frac{L_z \eta^2 T \tau}{2m} \sigma_z^2 \nonumber \\
    &\hspace{1cm} + \frac{L_y^2 \eta}{s} \sum_{t=0}^{T-1} \sum_{i=1}^{s} \sum_{j=0}^{\tau-1} \mathop{\mathbb{E}} \left[ \left\| \mathbf{y}_{t,j}^{i} - \mathbf{y}_t \right\|^2 \right] \nonumber\\
    &\hspace{1cm} + \frac{L_z^2 \eta}{2m} \sum_{t=0}^{T-1} \sum_{i=1}^{m} \sum_{j=0}^{\tau-1} \mathop{\mathbb{E}} \left[ \left\| \mathbf{z}_{t,j}^{i} - \mathbf{z}_t \right\|^2 \right] \nonumber
\end{align}
After a minor rearrangement, we have
\begin{align}
    \sum_{t=0}^{T-1} \mathop{\mathbb{E}}\left[ \left\| \nabla F(\mathbf{x}_t) \right\|^2 \right] & \leq \frac{2}{\eta \tau} \left( F(\mathbf{x}_0) - F(\mathbf{x}_{T-1}) \right) \nonumber \\
    & \hspace{1cm} + \frac{L_y\eta T}{s} \sigma_y^2 + \frac{2T}{s} \bar{\sigma}_y^2 + \frac{L_z \eta T}{m} \sigma_z^2 \nonumber \\
    &\hspace{1cm} + \frac{2 L_y^2}{s \tau} \sum_{t=0}^{T-1} \sum_{i=1}^{s} \sum_{j=0}^{\tau-1} \mathop{\mathbb{E}} \left[ \left\| \mathbf{y}_{t,j}^{i} - \mathbf{y}_t \right\|^2 \right] \nonumber \\
    & \hspace{1cm} + \frac{L_z^2}{m \tau} \sum_{t=0}^{T-1} \sum_{i=1}^{m} \sum_{j=0}^{\tau-1} \mathop{\mathbb{E}} \left[ \left\| \mathbf{z}_{t,j}^{i} - \mathbf{z}_t \right\|^2 \right] \nonumber
\end{align}
This completes the proof.
\end{proof}

\begin{lemma}
\label{lemma:discrepancy}
(model discrepancy) Under Assumption $1 \sim 3$, Algorithm 1 ensures
\begin{align}
    \frac{1}{s} \sum_{i=1}^{s} \sum_{j=0}^{\tau - 1} \mathop{\mathbb{E}} \left[ \left\| \mathbf{y}_{t,j}^{i} - \mathbf{y}_t \right\|^2 \right] \leq \frac{\eta^2 \tau (\tau - 1)}{1 - A_y} \sigma_y^2 + \frac{2 \tau A_y}{(1 - A_y)L_y^2} \bar{\sigma}_y^2 + \frac{2 \tau A_y}{(1 - A_y)L_y^2} \mathop{\mathbb{E}}\left[ \left\| \nabla_y F (\mathbf{x}_t) \right\|^2 \right], \nonumber
\end{align}
where $A_y \coloneqq 2\eta^2 L_y^2 \tau (\tau - 1) < 1$.
\end{lemma}
\begin{proof}
\begin{align}
    \mathop{\mathbb{E}} \left[ \left\| \mathbf{y}_{t,j}^{i} - \mathbf{y}_t \right\|^2 \right] &= \eta^2 \mathop{\mathbb{E}} \left[ \left\| \sum_{k=0}^{j-1} \nabla_y f(\mathbf{x}_{t,k}^{i}) \right\|^2 \right] \nonumber\\
    &= \eta^2 \mathop{\mathbb{E}} \left[ \left\| \sum_{k=0}^{j-1} \left( \nabla_y f(\mathbf{x}_{t,k}^{i}) - \nabla_y F_i(\mathbf{x}_{t,k}^{i}) + \nabla_y F_i(\mathbf{x}_{t,k}^{i}) \right) \right\|^2 \right] \nonumber \\
    &\leq 2\eta^2 \mathop{\mathbb{E}} \left[ \left\| \sum_{k=0}^{j-1} \left( \nabla_y f(\mathbf{x}_{t,k}^{i}) - \nabla_y F_i(\mathbf{x}_{t,k}^{i}) \right) \right\|^2 \right] + 2\eta^2 \mathop{\mathbb{E}} \left[ \left\| \sum_{k=0}^{j-1} \nabla F_i(\mathbf{x}_{t,k}^{i}) \right\|^2 \right] \label{discrepancy_jensen}\\
    &= 2\eta^2 \sum_{k=0}^{j-1} \mathop{\mathbb{E}} \left[ \left\| \nabla_y f(\mathbf{x}_{t,k}^{i}) - \nabla_y F_i(\mathbf{x}_{t,k}^{i}) \right\|^2 \right] + 2\eta^2 \mathop{\mathbb{E}} \left[ \left\| \sum_{k=0}^{j-1} \nabla_y F_i(\mathbf{x}_{t,k}^{i}) \right\|^2 \right] \label{discrepancy_special}\\
    &\leq 2\eta^2 (j-1) \sigma_y^2 + 2\eta^2 (j - 1) \sum_{k=0}^{j-1} \mathop{\mathbb{E}} \left[ \left\| \nabla_y F_i(\mathbf{x}_{t,k}^{i}) \right\|^2 \right] \label{discrepancy_jensen2}
\end{align}
where (\ref{discrepancy_jensen}) and (\ref{discrepancy_jensen2}) follows the convexity of $\ell_2$ norm and Jensen's inequality; (\ref{discrepancy_special}) holds because $\nabla_y f(\mathbf{x}_{t,k}^{i}) - \nabla_y F_i(\mathbf{x}_{t,k}^{i})$ has zero mean and is independent across $s$.

Then the second term on the right-hand side in (\ref{discrepancy_jensen2}) can be bounded as follows.
\begin{align}
    \mathop{\mathbb{E}} \left[ \left\| \mathbf{y}_{t,j}^{i} - \mathbf{y}_t \right\|^2 \right] &\leq 2\eta^2 (j-1) \sigma_y^2 + 2\eta^2 (j - 1) \sum_{k=0}^{j-1} \mathop{\mathbb{E}} \left[ \left\| \nabla_y F_i(\mathbf{x}_{t,k}^{i}) \right\|^2 \right] \nonumber \\
    & \leq 2\eta^2 (j-1) \sigma_y^2 + 4\eta^2 (j - 1) \sum_{k=0}^{j-1} \mathop{\mathbb{E}} \left[ \left\| \nabla_y F_i(\mathbf{x}_{t,k}^{i}) - \nabla_y F_i(\mathbf{x}_t) \right\|^2 \right] \nonumber\\
    &\hspace{1cm} + 4\eta^2 (j - 1) \sum_{k=0}^{j-1} \mathop{\mathbb{E}}\left[ \left\| \nabla_y F_i(\mathbf{x}_t) \right\|^2 \right] \nonumber\\
    & \leq 2\eta^2 (j-1) \sigma_y^2 + 4\eta^2 (j - 1) L_y^2 \sum_{k=0}^{j-1} \mathop{\mathbb{E}} \left[ \left\| \mathbf{y}_{t,k}^{i} - \mathbf{y}_t \right\|^2 \right] \nonumber\\
    &\hspace{1cm} + 4\eta^2 (j - 1) \sum_{k=0}^{j-1} \mathop{\mathbb{E}}\left[ \left\| \nabla_y F_i(\mathbf{x}_t) \right\|^2 \right] \nonumber
\end{align}

By summing up the above step-wise model discrepancy across all $\tau$ iterations within a communication round, we have
\begin{align}
    \sum_{j=0}^{\tau - 1} \mathop{\mathbb{E}} \left[ \left\| \mathbf{y}_{t,j}^{i} - \mathbf{y}_t \right\|^2 \right] & \leq \sum_{j=0}^{\tau - 1} 2\eta^2 (j-1) \sigma_y^2 + \sum_{j=0}^{\tau - 1} 4\eta^2 (j - 1) L_y^2 \sum_{k=0}^{j-1} \mathop{\mathbb{E}} \left[ \left\| \mathbf{y}_{t,k}^{i} - \mathbf{y}_t \right\|^2 \right] \nonumber \\
    & \hspace{1cm} + \sum_{j=0}^{\tau - 1} 4\eta^2 (j - 1) \sum_{k=0}^{j-1} \mathop{\mathbb{E}}\left[ \left\| \nabla_y F_i(\mathbf{x}_t) \right\|^2 \right] \nonumber \\
    & \leq \eta^2 \tau (\tau - 1) \sigma_y^2 + 4\eta^2 L_y^2 \sum_{j=0}^{\tau - 1} (j - 1) \sum_{k=0}^{j - 1} \mathop{\mathbb{E}} \left[ \left\| \mathbf{y}_{t,k}^{i} - \mathbf{y}_t \right\|^2 \right] \nonumber \\
    & \hspace{1cm} + 4\eta^2 \sum_{j=0}^{\tau - 1} (j - 1) \sum_{k=0}^{j-1} \mathop{\mathbb{E}}\left[ \left\| \nabla_y F_i(\mathbf{x}_t) \right\|^2 \right] \nonumber \\
    & \leq \eta^2 \tau (\tau - 1) \sigma_y^2 + 2\eta^2 L^2 \tau (\tau - 1) \sum_{j=0}^{\tau - 1} \mathop{\mathbb{E}} \left[ \left\| \mathbf{y}_{t,j}^{i} - \mathbf{y}_t \right\|^2 \right] \nonumber \\
    & \hspace{1cm} + 2\eta^2 \tau^2 (\tau - 1) \mathop{\mathbb{E}}\left[ \left\| \nabla_y F_i(\mathbf{x}_t) \right\|^2 \right]. \nonumber
\end{align}

After a minor rearranging, we have
\begin{align}
    \sum_{j=0}^{\tau - 1} \mathop{\mathbb{E}} \left[ \left\| \mathbf{y}_{t,j}^{i}  - \mathbf{y}_{t} \right\|^2 \right] \leq \frac{\eta^2 \tau (\tau - 1)}{1 - 2\eta^2 L^2 \tau (\tau - 1)} \sigma_y^2 + \frac{2\eta^2 \tau^2 (\tau - 1)}{1 - 2\eta^2 L_y^2 \tau (\tau - 1)} \mathop{\mathbb{E}}\left[ \left\| \nabla_y F_i(\mathbf{x}_t) \right\|^2 \right]. \label{long_localstep}
\end{align}

For the sake of simplicity, we define a constant $A_y \coloneqq 2\eta^2 L_y^2 \tau (\tau - 1) < 1$. Then, (\ref{long_localstep}) can be simplified as follows.
\begin{align}
    \sum_{j=0}^{\tau - 1} \mathop{\mathbb{E}} \left[ \left\| \mathbf{y}_{t,j}^{i}  - \mathbf{y}_{t} \right\|^2 \right] \leq \frac{\eta^2 \tau (\tau - 1)}{1 - A_y} \sigma_y^2 + \frac{\tau A_y}{(1 - A_y) L_y^2} \mathop{\mathbb{E}}\left[ \left\| \nabla_y F_i(\mathbf{x}_t) \right\|^2 \right]. \nonumber
\end{align}

Finally, averaging the above bound across all $s$ clients, we have
\begin{align}
    \frac{1}{s} \sum_{j=0}^{\tau - 1} \sum_{i=1}^{s} \mathop{\mathbb{E}} \left[ \left\| \mathbf{y}_{t,j}^{i}  - \mathbf{y}_{t} \right\|^2 \right] & \leq \frac{\eta^2 \tau (\tau - 1)}{1 - A_y} \sigma_y^2 + \frac{\tau A_y}{s(1 - A_y)L_y^2} \sum_{i=1}^{s} \mathop{\mathbb{E}}\left[ \left\| \nabla_y F_i(\mathbf{x}_t) \right\|^2 \right] \nonumber\\
    & = \frac{\eta^2 \tau (\tau - 1)}{1 - A_y} \sigma_y^2 + \frac{\tau A_y}{s(1 - A_y)L_y^2} \sum_{i=1}^{s} \mathop{\mathbb{E}}\left[ \left\| \left( \nabla_y F_i(\mathbf{x}_t) - \nabla_y F (\mathbf{x}_t) + \nabla_y F (\mathbf{x}_t) \right) \right\|^2 \right] \nonumber\\
    & \leq \frac{\eta^2 \tau (\tau - 1)}{1 - A_y} \sigma_y^2 + \frac{2 \tau A_y}{s(1 - A_y)L_y^2} \sum_{i=1}^{s} \mathop{\mathbb{E}}\left[ \left\| \left( \nabla_y F_i(\mathbf{x}_t) - \nabla_y F (\mathbf{x}_t) \right) \right\|^2 \right] \nonumber \\
    & \hspace{1cm} + \frac{2 \tau A_y}{s(1 - A_y)L_y^2} \sum_{i=1}^{s} \mathop{\mathbb{E}}\left[ \left\| \nabla_y F (\mathbf{x}_t) \right\|^2 \right] \nonumber\\
    & \leq \frac{\eta^2 \tau (\tau - 1)}{1 - A_y} \sigma_y^2 + \frac{2 \tau A_y}{(1 - A_y)L_y^2} \bar{\sigma}_y^2 + \frac{2 \tau A_y}{(1 - A_y)L_y^2} \mathop{\mathbb{E}}\left[ \left\| \nabla_y F (\mathbf{x}_t) \right\|^2 \right], \label{discrepancy_diff}
\end{align}
where (\ref{discrepancy_diff}) holds based on Assumption 3.
This completes the proof.
\end{proof}

\subsection {Other Lemmas and Proofs}
Herein, we show lemmas and their proofs used in the above analysis.

\begin{lemma}
\label{subgrad}
Under Assumption3, it is guaranteed that
\begin{align}
    \mathop{\mathbb{E}} \left[ \left\| \frac{1}{k} \sum_{i=1}^{k} \nabla F_i(\mathbf{x}) - \nabla F(\mathbf{x}) \right\|^2 \right] \leq \frac{\bar{\sigma}^2}{k}, \forall k \in \{1, \cdots, m\} \nonumber
\end{align}
\end{lemma}
\begin{proof}
\begin{align}
    \mathop{\mathbb{E}} \left[ \left\| \frac{1}{k} \sum_{i=1}^{k} \nabla F_i(\mathbf{x}) - \nabla F(\mathbf{x}) \right\|^2 \right] &= \frac{1}{k^2} \mathop{\mathbb{E}} \left[ \left\| \sum_{i=1}^{k} \nabla F_i(\mathbf{x}) - k \nabla F(\mathbf{x}) \right\|^2 \right] \nonumber\\
    &= \frac{1}{k^2} \mathop{\mathbb{E}} \left[ \left\| \sum_{i=1}^{k} \left( \nabla F_i(\mathbf{x}) - \nabla F(\mathbf{x}) \right) \right\|^2 \right] \nonumber\\
    &= \frac{1}{k^2} \sum_{i=1}^{k} \mathop{\mathbb{E}} \left[ \left\| \nabla F_i(\mathbf{x}) - \nabla F(\mathbf{x}) \right\|^2 \right] \label{lemma4:out}\\
    &\leq \frac{1}{k^2} \sum_{i=1}^{k} \bar{\sigma}^2 \label{lemma4:assum1} \\
    & = \frac{\bar{\sigma}^2}{k}, \nonumber
\end{align}
where (\ref{lemma4:out}) is because $\nabla F_i(\mathbf{x}) - \nabla F(\mathbf{x})$ is independent across all $i$ and thus $(\nabla F_i(\mathbf{x}) - \nabla F(\mathbf{x}))^{\top}(\nabla F_j(\mathbf{x}) - \nabla F(\mathbf{x})) = 0, \forall i \neq j$; (\ref{lemma4:assum1}) follows Assumption 3.
\end{proof}



\end{document}